\documentclass[twoside,11pt]{article}
\usepackage{mysty} 
\usepackage{fullpage}
\usepackage{amsfonts}
\usepackage{amssymb}
\usepackage{amstext}
\usepackage{amsmath}

\usepackage{mathtime}
\usepackage{graphicx}
\usepackage{epsfig}
\usepackage{enumitem}

\newcommand{\R}{\ensuremath{\mathbb R}}

\newcommand{\F}{\ensuremath{\mathbb F}}
\newcommand{\Filter}{{\mathbf F}}

\newcommand{\prob}[1]{\ensuremath{\text{{\bf Pr}$\left[#1\right]$}}}
\newcommand{\probb}[2]{\ensuremath{\text{{\bf Pr}$_{#1}\left[#2\right]$}}}
\newcommand{\expct}[1]{\ensuremath{\text{{\bf E}$\left[#1\right]$}}}

\newcommand{\expect}[2]{\ensuremath{\text{{\bf E}$_{#1}\left[#2\right]$}}}
\newcommand{\size}[1]{\ensuremath{\left|#1\right|}}

\newcommand{\poly}{\operatorname{poly}}

\newcommand{\ul}{\underline}

\newcommand{\p}{{\vec{p}}}

\newcommand{\silent}[1]{}

\newcommand{\DB}{{\mathbf D}}
\newcommand{\Ball}{{\mathbf B}}

\newcommand{\T}{{\mathcal T}}

\newcommand{\diff}{\textsf{diff}}
\newcommand{\inv}[1]{\frac{1}{#1}}
\newcommand{\abs}[1]{\left\lvert#1\right\rvert}
\newcommand{\twonorm}[1]{\left\lVert#1\right\rVert_2}
\newcommand{\norm}[1]{\left\lVert#1\right\rVert}
\newcommand{\ldot}{\ldots}

\newcommand{\score}{\textsf{score}}

\newcommand{\rscore}{\textsf{score}}

\def\reals{{\mathbb R}}

\def\E{{\mathcal E}}

\newenvironment{proofof}[1]{{\it Proof}{ of #1}.\hskip10pt}
{\hfill\rule{2mm}{2mm}\vskip5pt}
\firstpageno{1}

\begin{document}

\title{Learning Balanced Mixtures of Discrete Distributions \ \\ with Small Sample}
\author{\name Shuheng Zhou \email szhou@cs.cmu.edu \\
       \addr Computer Science Department\\
       Carnegie Mellon University\\
       Pittsburgh, PA 15213, USA}
\maketitle

\begin{abstract}
We study the problem of partitioning a small sample of $n$ individuals 
from a mixture of $k$ product distributions over a Boolean cube $\{0, 1\}^K$ 
according to their distributions. Each distribution is described by a 
vector of allele frequencies in $\R^K$.
Given two distributions, we use $\gamma$ to denote the average $\ell_2^2$ 
distance in frequencies across $K$ dimensions, 
which measures the statistical divergence 
between them. We study the case assuming that bits are independently distributed 
across $K$ dimensions.
This work demonstrates that, for a balanced input instance for $k = 2$,
a certain graph-based optimization function returns the correct partition with 
high probability, where a weighted graph $G$ is formed over $n$ individuals, 
whose pairwise hamming distances between their corresponding bit vectors
define the edge weights, so long as $K = \Omega\left(\ln n/\gamma\right)$ and 
$Kn = \tilde\Omega\left(\ln n/\gamma^2 \right)$.
The function computes a maximum-weight balanced cut of $G$, 
where the weight of a cut is the sum of the weights across all edges in the cut.
This result demonstrates a nice property in the high-dimensional feature
space: one can trade off the number of features that are required with the size 
of the sample to accomplish certain tasks like clustering.
\end{abstract}

\section{Introduction}
\label{sec:intro}
We explore a type of classification problem that arises in the context
of computational biology.  The problem is that we are given a small
sample of size $n$, e.g., DNA of $n$ individuals, each described by the values 
of $K$ {\em features} or {\em markers}, e.g., SNPs (Single Nucleotide Polymorphisms), 
where $n \ll K$.
Features have slightly different frequencies depending on which population the 
individual belongs to, and are assumed to be independent of each other.
Given the population of origin of an individual, the genotype (represented as
a bit vector in this paper) can be reasonably assumed to be generated by drawing 
alleles independently from the appropriate distribution.
The objective we consider is to minimize the number of features $K$, and thus
total data size $D = nK$, to correctly 
classify the individuals in the sample according to their population of origin, 
given any $n$. We describe $K$ and $nK$ as a function of the ``average quality'' 
$\gamma$ of the features. Throughout the paper, we use $p_i^j$ and $x_i^j$ as 
shorthands for $p_i^{(j)}$ and $x_i^{(j)}$ respectively. We first describe a general
mixture model that we use in this paper. The same model was previously used 
in~\citet{Zhou06} and ~\citet{BCFZ07}.

\medskip
\noindent{\bf Statistical Model:} We have 
$k$ probability spaces $\Omega_1,\ldots,\Omega_k$ over the set $\{0,1\}^K$.
Further, the components ({\em features}) of $z \in \Omega_t$ are independent and
$\probb{\Omega_t}{z_i=1}=p_t^i$ 
($1\leq t\leq k$, $1\leq i\leq K$).
Hence, the probability spaces $\Omega_1,\ldots,\Omega_k$ comprise
the distribution of the features for each of the $k$ populations.
Moreover, the input of the algorithm consists of  a
collection ({\em mixture}) of $n=\sum_{t=1}^kN_t$ unlabeled samples, $N_t$ points 
from $\Omega_t$, and the algorithm is to determine for each data point
from which of $\Omega_1,\ldots,\Omega_k$ it was chosen.
In general we do \emph{not} assume that $N_1,\ldots,N_t$ are revealed to
the algorithm; but we do require some bounds on their relative sizes.
An important parameter of the probability ensemble $\Omega_1,\ldots,\Omega_k$
is the {\em measure of divergence}
\begin{gather}
\label{eq:gamma}
\gamma = \min_{1\leq s<t\leq k}\frac{\sum_{i=1}^K(p^i_s - p^i_t)^2}{K}
\end{gather}
between any two distributions. 
Note that $\sqrt{K \gamma}$ provides a lower bound on the Euclidean distance 
between the means of any two distributions and represents their separation. 

Further, let $N=n/k$ (so if the populations were balanced we would have $N$ of
each type).
This paper proves the following theorem which gives a sufficient condition for a 
balanced ($N_1=N_2$) input instance when $k=2$.
\begin{theorem}\textnormal{\citep[Chapter~9]{Zhou06}}
\label{thm:intro-product-max-cut}
Assume $N_1=N_2=N$. If
$K = \Omega(\frac{\ln N} {\gamma})$ and
$KN = \Omega(\frac{\ln N \log \log N}{\gamma^2})$ then
with probability $1 - 1/\poly(N)$,
among all balanced cuts in the complete graph formed among $2N$ sample points, 
the maximum weight cut corresponds to the partition of the $2N$ points according to 
their distributions. Here the weight of a cut is the sum of weights across all 
edges in the cut, and the edge weight equals the Hamming distance between the bit 
vectors of the two endpoints.
\end{theorem}
Variants of the above theorem, based on a model that allows two random
draws at each dimension for all points, are given in~\citet[Theorem~3.1]{CHRZ07} and
~\citet[Chapter~8]{Zhou06}.
The cleverness there is the construction of a diploid score at each dimension, 
given any {\it pair of individuals}, under the assumption that two random bits can be 
drawn from the same distribution at each dimension.
In expectation, diploid scores are higher among pairs from different groups 
than for pairs in the same group across all $K$ dimensions.
In addition,~\citet[Lemma~2.2]{CHRZ07} shows that when $K > \Omega(\ln n/\gamma^2)$, 
given two bits from each dimension, one can always classify for any size of $n$, 
for unbalanced cases with any number of mixtures, using essentially connected 
component based algorithms, given the weighted graph as in described in 
Theorem~\ref{thm:intro-product-max-cut}.

The key contribution of this paper is to show new ideas that we use to 
accomplish the goal of clustering with the same amount of features, 
while requiring only one random bit at each dimension. 
While some ideas and proofs for Theorem~\ref{thm:intro-product-max-cut} 
in Section~\ref{sec:events} have appeared in~\citet{CHRZ07}, modifications for 
handling a single bit at each dimension are ubiquitous throughout the proof. 
Hence we contain the complete proof in this paper nonetheless to give a 
complete exposition.

Finding a max-cut is computationally intractable; 
a hill-climbing algorithm was given in~\citet{CHRZ07} to partition a balanced 
mixture, with a stronger requirement on $K$, given any $n$, 
as the middle green curve in Figure~\ref{fig:curves} shows.
Two simpler algorithms using spectral techniques were constructed 
in~\citet{BCFZ07}, attempting to reproduce conditions above.
Both spectral algorithms in~\citet{BCFZ07} achieve the bound established by
Theorem~\ref{thm:intro-product-max-cut} without requiring the input instances
being balanced, and work for cases when $k \geq 2$ is a constant;
However, they require $n = \Omega(1/\gamma)$, even when $k=2$ and the input instance 
is balanced, as the vertical line in Figure~\ref{fig:curves} shows.
Note that when $N = \tilde\Omega(1/\gamma)$, i.e., when we have enough sample from 
each distribution, $K = \Omega(\frac{\ln N} {\gamma})$ becomes the only requirement 
in Theorem~\ref{thm:intro-product-max-cut}.
Exploring the tradeoffs between $n$ and $K$, when $n$ is small, as in 
Theorem~\ref{thm:intro-product-max-cut} in algorithmic design is both of 
theoretical interests and practical value.
\begin{figure}[htb]
\begin{center}
\includegraphics[width=0.50\textwidth,angle=-90]{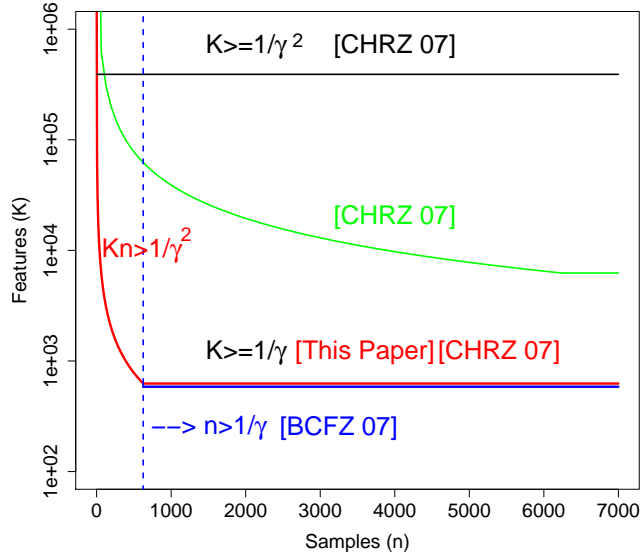}
\caption{This figure illustrates results from three papers. Top and middle curves
are algorithmic results from~\citet{CHRZ07}. Bottom red curve are non-algorithm
results from this paper with single random draw 
and~\citet{CHRZ07} with two random draws at each dimension. 
For $n > \Omega(1/\gamma)$, to the right of the vertical dashed 
line, spectral algorithms~\citet{BCFZ07} achieve bounds given in the red curve.
The curves are generated using a biased distribution in terms of the $\ell_1$ 
distances in allele frequencies: for $9/10$ of features, 
$\abs{p_1^i - p_2^i}= 10^{-5}$; and for the rest, it is $0.1265$; for this mixture,
$\gamma = 0.0016$.}
\label{fig:curves}
\end{center}
\end{figure}

\subsection{Related Work}
\label{sec:related}
In a seminal paper, ~\citet{PSD00} presented a model-based clustering method to 
separate populations using genotype data. 
They assume that observations from each cluster are random
from some parametric model. Inference for the parameters corresponding
to each population is done jointly with inference for the cluster membership
of each individual, and $k$ in the mixture, using Bayesian methods.

Applying spectral techniques by~\citet{Mcs01} on graph partitioning, 
and an extension due to~\citet{Coj06} from their original setting
on graphs to the asymmetric $n\times K$ matrix of individuals/features 
yields a polynomial time algorithm for this problem when $k$ is given as a 
constant, as analyzed by~\citet{BCFZ07}.
For $k=2$, an extremely simple algorithm based on examining values in the 
top two left singular vectors of the random matrix can cluster samples 
efficiently. However, spectral techniques require a lower bound on the 
sample size $n$ to be at least $1/\gamma$ as shown in 
Figure~\ref{fig:match-score}.

There are two streams of related work in the learning community.
The first stream is the recent progress in learning from the point of view
of clustering: given samples drawn from a mixture of
well-separated Gaussians (component distributions), one aims to
classify each sample according to which component distribution it comes
from, as studied in~\citet{Das99,DS00,AK01,VW02,AM05,KSV05,DHKS05}.
This framework has been extended to more general
distributions such as log-concave distributions
by~\citet{AM05,KSV05}, and heavy-tailed distributions by~\citet{DHKS05}, 
as well as to more than two populations.
These results focus mainly on reducing the requirement on the
separations between any two centers $P_1$ and $P_2$. In
contrast, we focus on the sample size $D$. This is motivated
by previous results~\citep{CHRZ07,Zhou06} stating that by acquiring enough
attributes along the same set of dimensions from each component distribution,
with high probability, we can correctly classify every individual.

While our aim is different from those results, where $n > K$ is almost
universal and we focus on cases $K > n$, we do have one common
axis for comparison, the $\ell_2$-distance between any two centers of the
distributions. In earlier works of~\citet{DS00,AK01}, the separation requirement 
depended on the number of dimensions of each distribution; this has recently 
been reduced to be independent of $K$, the dimensionality of 
the distribution for certain classes of distributions in~\citet{AM05,KSV05}.
This is comparable to our 
requirement in Theorem~\ref{thm:intro-product-max-cut} 
and that of~\citet{BCFZ07} for discrete distributions. 
For example, according to Theorem~$7$ in~\citet{AM05},
in order to separate the mixture of two Gaussians,
$\twonorm{P_1 - P_2} =
\Omega\left(\frac{\sigma}{\sqrt{\omega}} + \sigma \sqrt{\log n} \right)$
is required. 

Besides Gaussian and Logconcave, a general theorem 
in~\citet[Theorem~6]{AM05} is derived that in principle also applies to mixtures of
discrete distributions. 
The key difficulty of applying their theorem directly
to our scenario is that it relies on a concentration
property of the distribution~\citep[Eq (10)]{AM05} that need not hold in our case.
In addition, once the distance between any two centers is fixed, that is, once 
$\gamma$ is fixed in the discrete distribution, the sample size $n$
in their algorithms is always larger than
$\Omega\left(\frac{K}{\omega} \log^5 K\right)$~\citep{AM05,KSV05}
for log-concave distributions (in fact, in Theorem~$3$ of~\citet{KSV05},
they discard at least this many individuals in order to correctly classify
the rest in the sample), and larger than $\Omega(\frac{K}{\omega})$
for Gaussians~\citep{AM05}, whereas $n < K$ always holds when $n < \inv{\gamma}$
in the present paper.

The second stream of work is under the PAC-learning framework, where
given a sample generated from some target distribution $Z$, the goal
is to output a distribution $Z_1$ that is close to $Z$ in Kullback-Leibler 
divergence: $KL(Z||Z_1)$, where $Z$ is a mixture of product distributions 
over discrete domains or 
Gaussians~\citep{KMRRSS94,FM99,Cry99,CGG02,MR05,FOS05,FOS06}.
They do not require a minimal distance between any two
distributions, but they do not aim to classify every sample point
correctly either, and in general require much more data.

\section{Preliminaries and Definitions}
\label{product-intro}
Let us first formally define a product distribution over a Boolean cube
$\{0, 1\}^K$.
\begin{definition}
A product distribution $\DB_m, \forall m =1, 2$, over a Boolean cube 
$\{0, 1\}^K$ is characterized by its expected value 
$\p_m = (p^1_m, \ldots, p^K_m) \in [0, 1]^K$, which we refer to
as the center of $\DB_m$.
\end{definition}
We then restate our problem as a fundamental problem of learning 
mixtures of two product distributions over discrete domains, in particular,
over the $K$-dimensional Boolean cube $\{0, 1\}^K$, where $K$ is 
a variable whose value we need to resolve.
We use $X = \vec{x} = (x^1, x^2, \ldots, x^K)$ to represent a
random $K$-bit vector, given a set of $K$ attributes.
Sometimes we also use $x^i_j$ to represent the $i^{th}$ coordinate
of point $X_j$.
\begin{definition}
A random vector $\vec{x}$ from the distribution $\DB_m$, which
we denote as $\vec{x} \sim \DB_m$ or $\vec{x} \sim \p_m$, 
where $\p_m$ is the center of $\DB_m$, is generated by independently 
selecting each coordinate $x^i$ to be $1$ with probability 
$p^i_m$ and thus $\forall i, \forall m$,
$\expect{\vec{x} \sim \DB_m}{\vec{x}} = \vec{p}_m.$
\end{definition}
We next use the inner-product of two $K$-dimensional vectors
$\vec{x}$ and $\vec{y}$ as the $\score$ between $X$ and $Y$, as 
in Definition~\ref{def:product-score}, and define a complete graph,
where nodes are sample points and each edge weight is the $\score$ 
between the two endpoints. 
\begin{definition}
\label{def:product-score}
$\score(X, Y) = <\vec{x}, \vec{y}> = \sum_{i=1}^K x^i y^i$.
\end{definition}

\begin{definition}
\label{def:diff-x}
Let $X$ be a sample point from distribution $\DB_1$ 
and $Y$ be a sample point from $\DB_2$. Let $X'$, $Y'$ 
be points randomly drawn from $\DB_1$ and $\DB_2$ respectively,
\begin{eqnarray*}
\diff(X) & = & 
\expect{\vec{x'} \sim \vec{p}_1}{\rscore(X, X')} - 
\expect{\vec{y'} \sim \vec{p}_2}{\rscore(X, Y')}, \\
\diff(Y) & = & 
\expect{\vec{y'}\sim \vec{p}_2}{\rscore(Y, Y')} - 
\expect{\vec{x'} \sim \vec{p}_1}{\rscore(Y, X')},
\end{eqnarray*}
where expectations are taken over all possible 
realizations of $X'$, $Y'$ respectively.
\end{definition}

\section{The Approach}
\label{sec:approach}
Our goal is to show that the perfect partition  $\T = (P_1, P_2)$ is the 
minimum cut (min-cut) in terms of $\score$ 
among all balanced cut $(S, \bar{S})$, 
both in expectation and with high probability. 
Let us first define these objects formally. 
In this complete graph, let $P_1$ represent the set of points 
$X_1, X_2, \ldot, X_N$ from a product distribution $\DB_1$, 
and $P_2$ represent the set of points $Y_1, Y_2, \ldot, Y_N$  
from a product distribution $\DB_2$.
\begin{definition}
Consider a balanced cut $(S,\bar{S})$, as in Figure~\ref{fig:match-score}, where
$L \in [1, N/2]$ is the number of nodes that have been swapped from one 
side of $\T$ to the other, let
$S = \{X_i \in P_1, i = 1, \ldot, N-L, V_j \in P_2, j = 1, \ldot, L\}$, and
$\bar{S} = \{Y_i \in P_2, i= 1, \ldot, N-L, U_j \in P_1, j=1,\ldot, L\}$.
Let
$\rscore(S, \bar{S}) = \sum_{i=1}^{N-L} \sum_{j=1}^{N-L} \rscore(X_i, Y_j) + $ \\\
$\sum_{i=1}^L \sum_{j=1}^{L} \rscore(U_i, V_j) + 
\sum_{i=1}^{N-L} \sum_{j=1}^{L}\rscore(X_i, U_j)+ \rscore(Y_i, V_j),$
which defines $\rscore(\T)$ when $L = 0$, i.e.,
$\rscore(\T) = \sum_{i=1}^N \sum_{j=1}^{N} \rscore(X_i, Y_j).$
\end{definition}
It is easy to verify that in expectation, the perfect partition has
the minimum  $\rscore$, i.e., $\forall$ balanced $(S, \bar{S})$ other than
$\T$, that is, $\expct{\rscore(\T)}  < \expct{\rscore(S, \bar{S})}$). 
The following theorem says that this is true with high probability, given
a large enough $K$.
\begin{theorem}
\label{thm:product-min-cut}
For a balanced mixture of two distributions, with probability $1 - 1/\poly(N)$, 
$\rscore(\T) < \rscore(S, \bar{S})$, for all other balanced cut $(S, \bar{S})$, 
given $K = \Omega(\frac{\ln N} {\gamma})$ and 
$KN = \Omega(\frac{\ln N \log \log N}{\gamma^2})$, and $N \geq 4$.
\end{theorem}
\begin{corollary}
Following steps in Theorem~\ref{thm:product-min-cut}, one can show
that if scores are replaced with pairwise Hamming distances, i.e., 
$\forall X, Y,$ $H(\vec{x}, \vec{y}) = \sum_{i=1}^K x^i \oplus y^i$,
the max-cut will identify the perfect partition with high probability,
given the same order of number of attributes as stated in 
Theorem~\ref{thm:intro-product-max-cut}.
\end{corollary}
\def\sleft{\hskip-5pt}
\begin{figure}
\begin{center}
\begin{tabular}{cc}
\sleft
\includegraphics[width=2in, height=2in,angle=-0]{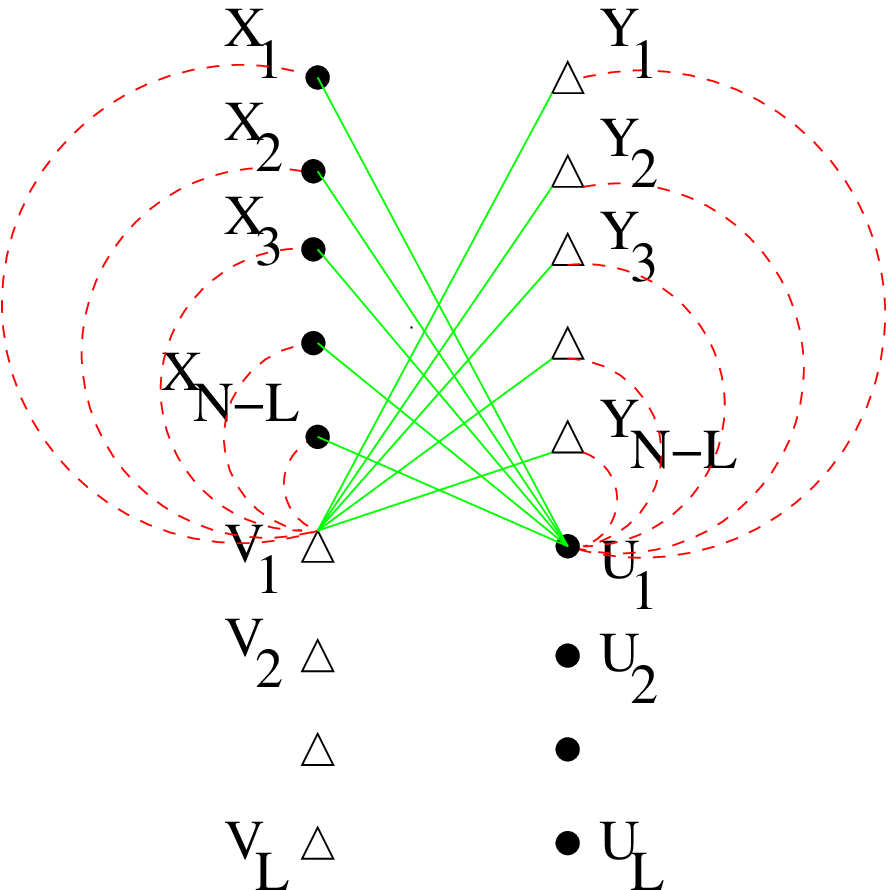} &
\includegraphics[width=2in, height=2in,angle=-0]{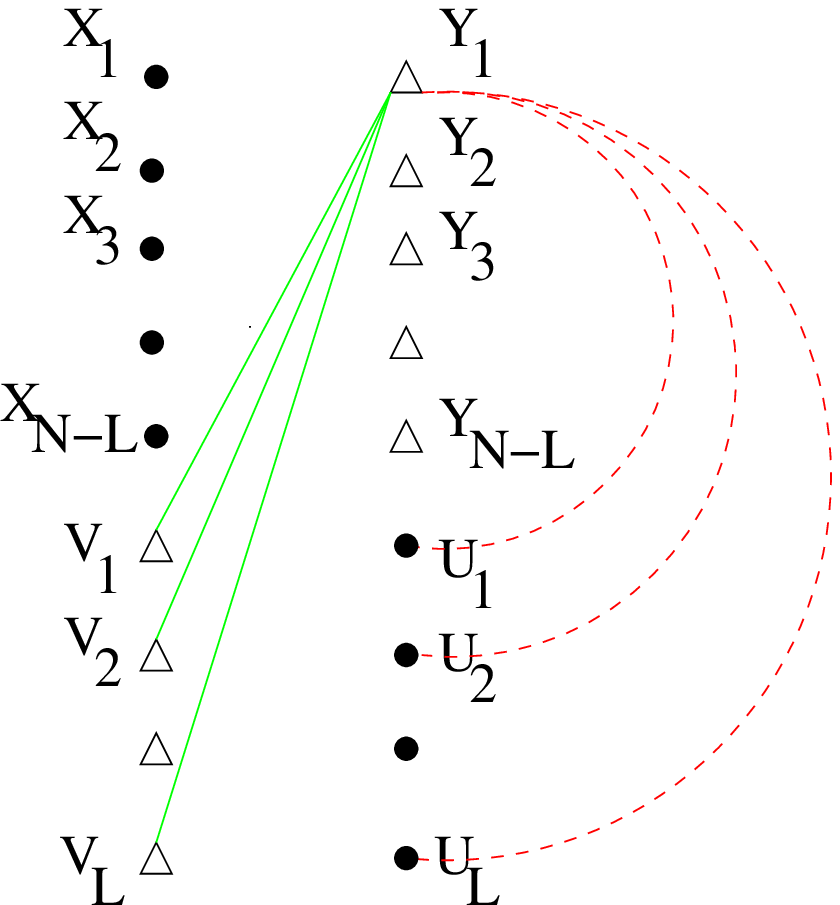}\\
\end{tabular}
\caption{Edges that are different between a perfect partition
$\T$ and another balanced partition $(S, \bar{S})$, seen only from
$U_1 \sim \p_1$ and $V_1 \sim \p_2$, and $Y_1 \sim \p_2$, 
red dotted edges are in $\T$ and green solid edges are in $(S, \bar{S})$.
In more detail, we refer to $X_i$ and $Y_i$, $\forall i \in [1, N-L]$ 
as  {\em unswapped} nodes, as the majority type in their side; we denote 
$V_j \in (S \cap P_2), U_j \in (\bar{S} \cap P_1), \forall j \in [1, L]$ 
as {\em swapped} nodes as the minority on their new side.
In particular, for $(S, \bar{S})$, original cut (red dotted) edges 
that belong to $\T$ are replaced with (green solid) edges,
which are the {\em new edges} that appear in $(S, \bar{S})$;
the set of common edges that belong to $\T \cap (S, \bar{S})$ are not
shown.}
\label{fig:match-score}
\end{center}
\end{figure}
The key technicality in this paper and~\citet{CHRZ07} is that,
instead of showing that each balanced cut $(S, \bar{S})$ has score
that is close to its expected value, we show that, for each balanced cut 
$(S, \bar{S})$, the following random variable $\diff(\T, (S,\overline{S}), L)$ 
as in~(\ref{eq:diff-score}), which captures the difference between the present 
cut and the unique perfect partition $\T$, stays close to its expected value, 
which is a positive number, given a large enough $K$. 
Note that for a particular balanced cut $(S, \bar{S})$, 
$\diff(\T, (S,\overline{S}), L) > 0$ immediately implies that 
$\rscore(\T) < \rscore(S, \bar{S})$.
Figure~\ref{fig:match-score} shows the edges whose weight contribute to:
\begin{eqnarray}
\label{eq:diff-score}
\lefteqn{\diff(\T, (S,\overline{S}), L) = \score(S, \bar{S}) - \score(\T) = } \\ \nonumber
& & \sum_{j=1}^L \sum_{i=1}^{N-L} \score(V_j, Y_i) - \score (V_j, X_i) + \score(U_j, X_i)- \score (U_j, Y_i).
\end{eqnarray}
The random variable 
$\diff(\T, (S,\overline{S}), L), \forall N/2 \geq L \geq 1$, comprises 
exactly of $\rscore$s over the set of edges that differ between those 
in $\T$ and those in $(S, \bar{S})$, which is exactly the 
set of $4 L (N-L)$ edges between swapped nodes and unswapped nodes,
among which $4(N-L)$ edges are shown in Figure~\ref{fig:match-score}.
Hence we only need to consider the influence 
of $2NK$ random bits over these two sets of edges contributing 
to~(\ref{eq:diff-score}), $\forall (S, \bar{S})$.
It is not hard to verify the following:
\begin{equation}
\label{eq:exp-diff}
\expct{\diff(\T, (S,\overline{S}), L)} 
=  
(N-L)L\left(\expect{\vec{x}\sim \vec{p}_1}{\diff(X)} +
 \expect{\vec{y}\sim \vec{p}_2}{\diff(Y)}\right).
\end{equation}
\subsection{Key Idea in the One-bit Construction}
The difference from~\citet{CHRZ07} is that we require only a single bit 
at each dimension for $\score$ in the present paper. 
The idea that makes an inner-product based score work is that
although from an individual, e.g., $Y$'s perspective, $\diff(Y)$ 
may not be significantly positive due to the definition of our 
$\rscore$, the sum of $\diff$s over a pair of swapped nodes, 
e.g., $\diff(X) + \diff(Y)$ as in Figure~\ref{fig:diffx}, can be
shown to be positive with high probability, given $K = \Omega(\ln N /\gamma)$.
Hence we prevent the sum of $\diff(X) + \diff(Y)$ from deviating too much 
from its expected value $K \gamma$ (Proposition~\ref{pro:fm99}), by 
excluding those {\it bad node events} (Definition~\ref{def:one-dim-p}), 
whose probability we bound in Lemma~\ref{lemma:mu-x-bound-p} 
and~\ref{lemma:mu-y-bound}.
\begin{definition}{\bf(Bad Node Event)}
\label{def:one-dim-p}
Let a {\em bad node} event $\E(Z)$ be the event that 
$\{\diff(Z) < \expct{\diff(Z)} - K \gamma/4\}$, 
where $Z$ is a sample point in the mixture.
Note this is an event in an individual probability space 
$(\Omega_Z, \F_Z, {\bf Pr}_Z)$, where 
$(\Omega_Z, \F_Z, {\bf Pr}_Z)$ is defined over all possible
outcomes of $K$ random bits for sample point $Z$.
\end{definition}
Note that all bad node events are mutually independent.
From now on, we use $(\Omega_i, \F_i, {\bf Pr}_i)$ to 
refer to $(\Omega_{Z_i}, \F_{Z_i}, {\bf Pr}_{Z_i})$ 
for the input $2N$ nodes, assuming a certain ordering.
\begin{definition}{\bf{(Bad Event $\E^N_1$)}}
\label{def:product-space}
$\E^N_1$ is the same as $\E(Z_1) \cup \ldots \cup \E(Z_{2N})$ in
the product probability space $(\Omega, \F, {\bf Pr})$ composed of 
distinct probability spaces 
$(\Omega_1, \F_1, {\bf Pr}_1), \ldots$, $(\Omega_{2N}, \F_{2N}, {\bf Pr}_{2N})$ 
as in Definition~\ref{def:one-dim-p}. Let $\bar\E^N_1$ denote the product 
probability space $(\Omega, \F, {\bf Pr})$ excluding ${\E}^N_1$.
\end{definition}
\begin{figure}[htb]
\begin{center}
\psfig{figure=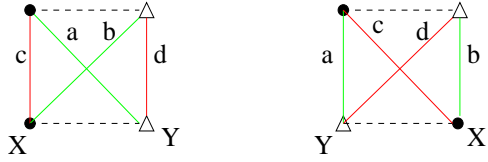, width=2.5in,height=0.8in}
\caption{Given Dots$\sim \p_1$ and Triangles$\sim \p_2$.
Define $\diff(X) = \expct{c|X} - \expct{b|X}$ and 
$\diff(Y) = \expct{d|Y} - \expct{a|Y}$. Given
$K = \Omega(\ln N /\gamma)$, with high probability, 
$\diff(X) + \diff(Y) \geq K \gamma/2$, 
given that 
$\expect{\vec{x}\sim \vec{p}_1}{\diff(X)} + 
\expect{\vec{y}\sim \vec{p}_2}{\diff(Y)} = K \gamma$;
Hence $a + b \leq c + d$, with high probability, 
given also that $KN = \Omega(\ln N \log \log N /\gamma^2)$.}
\label{fig:diffx}
\end{center}
\end{figure}
For each balanced cut $(S, \overline{S})$, conditioned upon fixing a subset of 
random bits on all swapped nodes, as shown in Figure~\ref{fig:match-score}, 
to behave nicely in the sense of Lemma~\ref{lemma:mu-x-bound-p} 
and~\ref{lemma:mu-y-bound},
we show that the conditional expectations, in the sense of Definition~\ref{def:e_h},
for random variables $\diff(\T, (S,\overline{S}), L)$, $\forall L > 0$, 
are significantly positive, so that the perfect partition can almost always win 
over all other balanced cuts, in terms of the particular measure 
(minimum total score here), despite the large deviation events that we handle 
in Section~\ref{sec:events}. 
This idea has been explored in the proof of~\citet{CHRZ07} for diploid scores.

The key difference between this score and the 
``diploid score''~\citep[see][Section~2.1]{CHRZ07} is that the corresponding 
diploid $\diff(Y)$ is always {\it significantly positive} in expectation, i.e., 
$\expect{\vec{y} \sim \DB_m}{\diff(Y)} > 0$, $\forall m =1, 2$, and thus 
remains so with high probability given $K = \Omega(\ln N /\gamma)$. 
That is, an individual is almost always more similar to a randomly chosen 
peer from its population, than a randomly chosen individual from another 
population given a large enough $K$ based on ``diploid scores''. 
The cost of this nice property is: two random bits from the same distribution 
are required at each dimension from all sample.
In the present paper, we provide a similar positiveness guarantee, 
for a pair of scores
$\diff(X) + \diff(Y)$, where $\vec{x}\sim \DB_1$ and $\vec{y}\sim \DB_2$, 
as illustrated in Figure~\ref{fig:diffx}. This property is due to 
Proposition~\ref{pro:fm99}, Lemma~\ref{lemma:mu-x-bound-p} 
and~\ref{lemma:mu-y-bound}. We like to point out that the requirement on the input
instance being balanced is due to the fact that we need pairing up two individuals
such that one comes from each distribution, in order to obtain the initial expected
minimality for $\T$ as defined in Proposition~\ref{pro:adv-max-intro}.

\subsection{The Expected Difference of Two Edges}
\label{sec:one-bit-prop}
We first show that the perfect partition $\T$ has the minimum value 
among all balanced cuts {\it in expectation}, when summing up scores over all edges 
across the cut in Proposition~\ref{pro:adv-max-intro}.
The inspiration for using an inner-product based score and pairing
up $\diff(X)$ and $\diff(Y)$, for $X \sim \DB_1$ and $Y \sim \DB_2$,
comes from~\citet{FM99}.
We first show that the sum of expected differences over 
$X \sim \DB_1$ and $Y \sim \DB_2$ is significant.
\begin{proposition}
$\forall a, b = 1, 2, 
\expect{\vec{x} \sim \DB_a, \vec{y} \sim \DB_b}{<\vec{x}, \vec{y}>} =
<\vec{p}_a, \vec{p}_b>$.
\end{proposition}

\begin{proof}
We have $\forall a, b = 1, 2$, 
$\expect{\vec{x} \sim \DB_a, \vec{y} \sim \DB_b}{<\vec{x}, \vec{y}>} 
= \expct{\sum_{i=1}^K x^i y^i} =  \sum_{i=1}^K \expct{x^i y^i} 
= \sum_{i=1}^K p^i_a p^i_b = <\vec{p}_a, \vec{p}_b>.$ 
\end{proof}

\begin{proposition}
\label{pro:diff-val}
Let $X$ be a sample point from $\DB_1$ 
and $Y$ be a point from $\DB_2$,
$\diff(X) =
\sum_{i=1}^K x^i (p^i_1 - p^i_2),$ and
$\diff(Y) = \sum_{i=1}^K y^i (p^i_2 - p^i_1)$.
\end{proposition}

\begin{proposition}{\textnormal{\citep{FM99}}}
\label{pro:fm99}
$\expect{\vec{x}\sim \vec{p}_1}{\diff(X)} + 
\expect{\vec{y}\sim \vec{p}_2}{\diff(Y)} = 
\norm{\vec{p}_1 - \vec{p}_2}_2^2 = K \gamma$.
\end{proposition}
\begin{proof}
By Proposition~\ref{pro:diff-val},
$\expect{\vec{x} \sim \vec{p}_1}{\diff(X)} + 
\expect{\vec{y}\sim \vec{p}_2}{\diff(Y)}  
 =  \sum_{i=1}^K p^i_1 (p^i_1 - p^i_2) +
\sum_{i=1}^K p^i_2 (p^i_2 - p^i_1) =  <\vec{p}_1, \vec{p}_1 - \vec{p}_2> +  
<\vec{p}_2, \vec{p}_2 - \vec{p}_1>  
= K \gamma.$
\end{proof}
Before we proceed, we first state the following theorem and its corollary 
on Hoeffding Bounds.
\begin{theorem}{\textnormal{\citep{Hoe63}}}
\label{thm:hoe-bound}
If $X_1, X_2, \ldot, X_K$ are independent and $a_i \leq X_i \leq b_i, 
\forall i = 1, 2, \ldot, K$, and if 
$\bar{X} = (X_1 + \ldot + X_K)/K$ and $\mu = \expct{\bar{X}}$,
then for $t > 0$,
$\prob{\bar{X} - \mu \geq t} \leq e^{-2K^2 t^2/\sum_{i=1}^K (b_i -a_i)^2}.$
\end{theorem}
\begin{corollary}{\textnormal{\citep{Hoe63}}}
\label{coro:hoe-diff}
If $Y_1, \ldot, Y_n$, $Z_1, \ldot, Z_m$ are independent random variables
with values in the interval $[a, b]$, and if $\bar{Y} = (Y_1 + \ldot + Y_m)/m$,
$\bar{Z} = (Z_1 + \ldot + Z_n)/n$, then for $t > 0$,
$$\prob{\bar{Y} - \bar{Z} - (\expct{\bar{Y}} - \expct{\bar{Z}}) \geq t} 
\leq e^{-2 t^2/(m^{-1} + n^{-1}) (b -a)^2}.$$
\end{corollary}
Let us denote w.l.o.g. 
$\eta = \expect{\vec{x}\sim \vec{p}_1}{\diff(X)} \geq K \gamma/2$,
and thus $\expect{\vec{y}\sim \vec{p}_2}{\diff(X)} = K \gamma - \eta$, and 
show the following two lemmas.
\begin{lemma}
\label{lemma:mu-x-bound-p}
Given that $K \geq \frac{8 \ln {1/\tau}}{\gamma}$,
$\probb{X}{\diff(X) < \eta - K \gamma/4} < \tau$.
\end{lemma}
\begin{proof}
Let us define 
$\gamma_k = (p_1^k - p_2^k)^2, \forall k =1, \ldots, K$.
Given that $x^1, \ldots, x^K$ are independent Bernoulli random 
variables and $(p^k_1 -p^k_2) x^k$ is either in $[0, \sqrt{\gamma_k}]$ 
or $[-\sqrt{\gamma_k}, 0]$, $\forall k = 1, \ldot, K$, we apply
Hoeffding bound as in Theorem~\ref{thm:hoe-bound} with
$t = K \gamma/4K = \gamma/4$:
\begin{eqnarray*}
\probb{X}{-\sum_{k=1}^K (p^k_1 - p^k_2)x^k + \eta \geq K \gamma/4}& =& 
\probb{X}{\sum_{k=1}^K (p^k_1 - p^k_2) x^k - \eta \leq - K \gamma/4} \\
& \leq & 
e^{-2K^2 (\gamma/4)^2/\sum_{k=1}^K{(\sqrt{\gamma_k})^2}}  \leq  \tau.
\end{eqnarray*}
Thus we have that
$\probb{X} {\sum_{k=1}^K (p^k_1 - p^k_2) x^k \geq \eta - K \gamma/4}
\geq 1 - \tau$.
\end{proof}

\begin{lemma}
\label{lemma:mu-y-bound}
Given that $K \geq \frac{8 \ln {1/\tau}}{\gamma}$,
$\probb{Y}{\diff(Y) < (K \gamma - \eta) - K \gamma/4} < \tau$.
\end{lemma}
\begin{proof}
Similar to proof of Lemma~\ref{lemma:mu-x-bound-p}, we have $\probb{Y}
{\sum_{k=1}^K (p^k_2 - p^k_1) y^i - (K \gamma - \eta) \leq - K \gamma/4}
\leq \tau$, where $K\gamma -\eta = \expect{\vec{y} \sim \vec{p}_2}{\diff(Y)}$. Hence
$[\probb{Y}
{\sum_{k=1}^K (p^k_2 - p^k_1) y^i \geq (K \gamma - \eta) - K \gamma/4}
\geq 1- \tau.$
\end{proof}
In particular, combining~$(\ref{eq:exp-diff})$ and Proposition~\ref{pro:fm99}, 
we have the following.
\begin{proposition}
\label{pro:adv-max-intro}
$\expct{\diff(\T, (S,\overline{S}), L)} = (N-L)L K \gamma$.
\end{proposition}
\begin{proof}
By Definition~\ref{def:diff-x}, we have
\begin{eqnarray*}
\lefteqn{\diff(X) = 
\expect{\vec{x'} \sim \vec{p}_1}{\rscore(X, X')} -
\expect{\vec{y'} \sim \vec{p}_2}{\rscore(X, Y')}}\nonumber\\
& = & 
\expect{\vec{x'} \sim \vec{p}_1}{<\vec{x}, \vec{x'}>} -
\expect{\vec{y'} \sim \vec{p}_2}{<\vec{x}, \vec{y'}>} 
 =  <\vec{x}, \vec{p}_1 - \vec{p}_2> =
\sum_{i=1}^K x^i (p^i_1 - p^i_2), \\
\lefteqn{\diff(Y) =
\expect{\vec{y'} \sim \vec{p}_2}{\rscore(Y, Y')} -
\expect{\vec{x'} \sim \vec{p}_1}{\rscore(Y, X')}} \nonumber \\
&= & 
\expect{\vec{y'} \sim \vec{p}_2}{<\vec{y}, \vec{y'}>} -
\expect{\vec{x'} \sim \vec{p}_1}{<\vec{y}, \vec{x'}>} 
 = <\vec{y}, \vec{p}_2 - \vec{p}_1> =
\sum_{i=1}^K y^i (p^i_2 - p^i_1).
\end{eqnarray*}
\end{proof}
Given such a positiveness guarantee on the conditional expectations of 
$\diff(\T, (S,\overline{S}), L)$ described above, the rest of the proof focus 
on bounding large deviation events; a sketch of the key ideas has appeared 
in~\citet[Section~3]{CHRZ07}, based on ``diploid scores''.
We need to show that, with high probability, all of $O(2^n)$ random variables, 
in the form of $\diff(\T, (S,\overline{S}), L)$, stay positive all simultaneously, 
given enough number of features and total number of random bits.
We describe the important ideas of this proof in next three sections, which contain
key lemmas for each step; more proofs are contained in the appendix for 
completeness of presentation.

\section{Proof Techniques for Concentration}
\label{sec:events}
We first introduce some notation regarding
the sample probability space $(\Omega, \F, {\bf Pr})$.
The set $\Omega$ is the set of all possible outcomes for $2NK$ 
random bits, where we denote each bit as $b^k_j$ for a point $j$ at 
dimension $k$.
The $\sigma$-field $\F$ of events is the set $\Sigma(\Omega)$ of all
subsets of $\Omega$; and the probability measure ${\bf Pr}$ is based 
on the product of probabilities of each random bit $b^k_j, \forall k, j$, 
corresponding to Bernoulli($p^k_{a}$), where $a \in \{1, 2\}$ depends 
on the population of origin for individual $j$. 
Formally,
\begin{definition}
The elementary events in the underlying sample space 
$(\Omega, \F, {\bf Pr})$ are all possible $2^{2NK}$
choices of $D = 2NK$ bits.
For $0\leq i \leq D$ and $w \in \{0, 1\}^i$,
let $B_w$ denote the event that the first $i$ bits
equal to the bit string $w$.
Let $\F_{i}$ be the $\sigma$-field generated by the partition of 
$\Omega$ into blocks $B_w$, for $w \in \{0, 1\}^i$. 
Then the sequence $\F_0, \ldots, \F_D$ forms a filter. 
In the $\sigma$-field $\F_i$, the only valid events are the ones that 
depend on the values of the first $i$ bits, and all such events 
are valid within.
\end{definition}
The events that we define next and their interactions 
are shown in Figure~\ref{fig:events}.
We show that, with high probability, all of the $O(2^{2N})$ 
random variables $\diff(\T, (S,\overline{S}), L)$, as
in~(\ref{eq:diff-score}), one corresponding to each 
balanced $(S, \bar{S})$, are positive.
\silent{
In particular, observe that all random variables, 
$\diff(\T, (S,\overline{S}), L),\forall (S, \bar{S})$, 
$\forall L> 0$ have positive expected values, as in 
Proposition~\ref{pro:adv-max-intro}, as their initial {\em advantage}; 
thus we need to show  that, with high probability, for all balanced cuts, 
the deviation of its corresponding random variable is less than the 
expected advantage.
}
We initially confine ourselves 
into a {\em good} subspace $\bar{\E}^N_1$ by excluding any {\em bad node} event
(Definition~\ref{def:one-dim-p}).
This subspace has the nice property in the sense of Theorem~\ref{thm:exp-h-p}.
We then use union bound to bound the probability of any {\em bad score} event in
this subspace, where a single {\em bad score} event occurs when 
$\diff(\T, (S,\overline{S}), L) \leq 0$ for a particular 
balanced $(S, \bar{S})$.
We  use the bounded differences method to bound probabilities of such events.

Each time we examine $\diff(\T, (S,\overline{S}), L)$ for a particular
balanced $(S, \bar{S})$, 
we let vector $(H_1, \ldots, H_{2KN})$
record the entire history of random bits,
where $(H_1, \ldots, H_{2KL})$ record the partial history of 
bits on the $2L$ swapped nodes corresponding
to $(S, \bar{S})$.
Let $\ell = 2KL$ be a positive integer.
We denote this $2KL$-history with $\ul{H}^{(\ell)}$.
For a balanced $(S, \bar{S})$, let $\ul{h}$
be a fixed possible $\ell$-history:
$\ul{h} = \{\tilde{U_1}, \ldots, \tilde{U_L}, \tilde{V_1}, 
\ldots, \tilde{V_L}\}$ 
denotes a vector of $2KL$ random bits 
on $2L$ swapped nodes as shown in Figure~\ref{fig:match-score}, 
where $\tilde{X}$ is the outcome of a particular point $X$ in our sample.
Let $\Omega_{\ul{h}}$ denote that event that we observe this particular 
$2KL$-history: $\Omega_{\ul{h}} = \{\pi \in \Omega: H^{(\ell)}(\pi) = \ul{h}\}$.
Given that $\Omega_{\ul{h}}$ occurs, we are concerned about the 
following probability space 
$(\Omega_{\ul{h}}, \Sigma(\Omega_{\ul{h}}), {\bf Pr}_{\ul{h}})$, we have
the following definition and proposition.
\begin{definition}
\label{def:e_h}
$\expect{\ul{h}}{\diff(\T, (S,\overline{S}), L)} =
\expct{\diff(\T, (S,\overline{S}), L)|\F_{2KL}}$ is the expected value 
of $\diff(\T, (S,\overline{S}), L)$ conditioned on an event
$\ul{h} \in \F_{2KL}$. This conditional expectation
$\expct{\diff(\T, (S,\overline{S}), L)|\F_{2KL}}$
is a random variable that can be viewed as a function into $\reals$ 
from the blocks in the partition of $\F_{2KL}$.
\end{definition}
Hence $\expect{\ul{h}}{\diff(\T, (S,\overline{S}), L)}$ is an evaluation
at a particular outcome $\ul{h} \in \F_{2KL}$.
\begin{proposition}
\label{pro:history-p}
For a particular outcome $\ul{h} \in \F_{2KL}$, 
$\expect{\ul{h}}{\diff(\T, (S,\overline{S}), L)} = 
(N-L)\sum_{j=1}^{L} \diff(\tilde{U}_j) + 
(N-L)\sum_{j=1}^{L} \diff(\tilde{V}_j) = 
(N-L) \sum_{j=1}^L \sum_{k=1}^K (p^k_1 - p^k_2)(\tilde{u}^k_j - \tilde{v}^k_j).$
\end{proposition}
Our starting point for using the bounded differences method to bound a single
{\em bad score} event over $(S, \bar{S})$ is when we have revealed the $2KL$ bits 
and obtained a $2KL$-history $\ul{h}$ in $\bar{\E}^N_1$.
Given a fixed history $\ul{h}$, we call the remaining $2K(N-L)$ bits 
on unswapped nodes as the $2K(N-L)$-{\em future}. 
Let $\bar{f} = (H_{2KL+1}, \ldots, H_{2KN})$
be a fixed possible $2K(N-L)$-{\em future}.
For simplicity of analysis, given $\ul{h}$, we first {\em expand} the 
confined subspace $\bar{\E}^N_1$ by dropping 
constraints on the $2(N-L)$ unswapped nodes.

In this expanded subspace, we only require the first $2L$ swapped nodes to be 
{\em good} nodes, a condition that we denote with $\bar{\E}^L_1(S, \bar{S})$, 
while leaving bits on the $2(N-L)$ unswapped nodes unconstrained; that is, 
these nodes can be {\em bad} nodes. 
Thus $(\Omega_{\ul{h}}, \Sigma(\Omega_{\ul{h}}), {\bf Pr}_{\ul{h}})$
corresponds to the {\em expanded} subspace of $\bar{\E}^N_1$ given 
$\ul{h}$, where we can apply the bounded differences method to analyze 
probability for $\{\diff(\T, (S,\overline{S}), L) \leq 0\}$ in a 
clean manner applying Azuma's Inequality as in Lemma~\ref{lemma:azuma}.
In fact, our starting point of the bounded differences analysis is
$\expect{\ul{h}}{\diff(\T, (S,\overline{S}), L)}$, where 
$\ul{h}$ is a fixed possible $2KL$-history on the $2L$ swapped nodes for 
$(S, \bar{S})$, subject to $\ul{h} \in \bar{\E}^L_1(S, \bar{S})$:
\begin{definition}
\label{def:ell}
$\E^L_1(S, \bar{S})$ is the same as $\E(U_1) \cup \ldots \cup \E(U_{L}) 
\cup \E(V_1) \cup \ldots \cup \E(V_{L})$ in 
the product probability space composed of distinct probability spaces 
defined over nodes $U_1, \ldots, U_L, V_1, \ldots, V_L$ as in 
Definition~\ref{def:one-dim-p}.
\end{definition}
This immediately indicates that the conditional expected value
$\expect{\ul{h}}{\diff(\T, (S,\overline{S}), L)} \geq (N-L)LK\gamma/2$,
which is our ``advantageous base point'' given that $\Omega_{\ul{h}}$ 
occurs. The proof of the following theorem appears in Section~\ref{sec:cond-exp}.
\begin{theorem}
\label{thm:exp-h-p}
Give that all points are drawn from $\bar{\E}^N_1$, 
the probability space $(\Omega, \F, {\bf Pr})$ excluding ${\E}^N_1$, 
we have $\forall$ balanced $(S, \bar{S})$,
where $\ul{h}$ is a particular $2KL$-history corresponding to 
the $2L$ swapped nodes specified over $(S, \bar{S})$ with 
respect to $\T$,
\begin{gather}
\label{eq:exp-history1}
\expect{\ul{h}}{\diff(\T, (S,\overline{S}), L)}
\geq  (N-L)L K \gamma/2,
\end{gather}
where the conditional expectation is over each of the individually 
expanded probability space
$(\Omega_{\ul{h}}, \Sigma(\Omega_{\ul{h}}), {\bf Pr}_{\ul{h}})$
given $\ul{h} \in \bar{\E}^L_1$, where $\E^L_1$ is defined in
Definition~\ref{def:ell}.
This statement remains true after we require that
$\ul{h} \in \bar{\E}^L_2$ in addition, where $\E^L_2$ is defined 
in Definition~\ref{def:big-delta}.
\end{theorem}

Now as we reveal one by one the future $2K(N-L)$ random bits, 
the conditional expected values  
$\expect{\ul{h}}{\diff(\T, (S,\overline{S}), L)|\ul{H}^{(\ell')}},
\forall \ell' \geq 2KL$ form a martingale that is amenable to 
the bounded differences analysis as shown in Theorem~\ref{thm:diff-dev-p} 
in Section~\ref{sec:bounded-diff}. However, in order to obtain a 
concentration bound as tight as that in Theorem~\ref{thm:diff-dev-p}, 
we need to exclude one more event $\E^L_2$
as in Definition~\ref{def:big-delta}, from the $2KL$-history $\ul{h}$, 
while examining a balanced $(S, \bar{S})$. We first give some definitions 
regarding $\E^L_2$. Nodes are shown in Figure~\ref{fig:match-score}.
\begin{definition}
\label{def:f2-product}
Given vectors $\vec{u_1}, \ldots, \vec{u_L}$ and 
$\vec{v_1}, \ldots, \vec{v_L}$, 
where $u^k_j, v^k_j$ are the $k^{th}$ bit of $U_j$ and $V_j$ respectively,
$f_2^k(\ul{h}) 
= \sum_{j=1}^L u^k_j - \sum_{j=1}^L v^k_j.$
\end{definition}
\begin{definition}{\bf{(Deviation Values)}}
\label{def:devi-p}
$\forall k = 1, \ldots, K$, let $t_k\sqrt{L}$ be 
the {\em exact} deviation on $f_2^k(\ul{h})$, i.e.,
$f^k_2(\ul{h}) - \expct{f^k_2(\ul{h})} = t_k \sqrt{L}, \forall k$.
\end{definition}
\begin{definition}{\bf{(Bad Deviation Event $\E^L_2$)}}
\label{def:big-delta}
In probability space $(\Omega, \F, {\bf Pr})$, 
given a balanced $(S, \bar{S})$ and its corresponding $2KL$-history 
$\ul{h}$, $\E^L_2$ is the event such that the set of random variables 
$t_1, \ldot, t_k$ regarding $2KL$ random bits recorded in $\ul{h}$, 
as defined in Definition~\ref{def:devi-p}, are 
{\em simultaneously large} and satisfy
$\sum_{k=1}^K {t}^2_k \geq \Delta = 
8N \ln2 + 4K \ln2 (\log \log N + 1) + 3\ln N/2$.
\end{definition}
Using Definition~\ref{def:big-delta} and~\ref{def:devi-p},
we immediately have the following lemma.
\begin{lemma}
\label{lemma:diff-bound-p}
Given that $\ul{h} \in \bar{\E}_2^L$, we have $\forall k$,
\begin{gather*}
\size{f^k_2(\ul{h})} \leq 
\size{\expct{f^k_2(\ul{h})}} + \size{t_k \sqrt{L}}, 
\end{gather*}
and $\sum_{k=1}^K {t}^2_k \leq \Delta$,
where $t_k$ is in Definition~\ref{def:devi-p}, and $\E^L_2$ is 
in Definition~\ref{def:big-delta}.
\end{lemma}
\begin{proof}
By definition of $t_k, \forall k$, we have that
$f^k_2(\ul{h}) = \expct{f^k_2(\ul{h})} + t_k \sqrt{L}$,
where $t_k \in [\frac{-L - \expct{f^k_2(\ul{h})}}{\sqrt{L}}, 
\frac{L -\expct{f^k_2(\ul{h})}}{\sqrt{L}}]$.
Thus the lemma holds given that $\ul{h} \in \bar{\E}_2^L$. 
\end{proof}
Excluding $\E^L_2$ from $\ul{h}$ is crucial in bounding
the difference that each of the $2(N-L)K$-{\em future} random bits 
causes when we work in probability space
$(\Omega_{\ul{h}}, \Sigma(\Omega_{\ul{h}}), {\bf Pr}_{\ul{h}})$, 
where the {\em difference} refers to
$$\size{
\expect{\ul{h}}{\diff(\T, (S,\bar{S}), L)| \ul{H}^{(\ell')}}-
\expect{\ul{h}}{\diff(\T, (S,\overline{S}), L) | \ul{H}^{(\ell'-1)}}},$$
where $2KN \geq \ell' > 2KL$ depends on the bit, 
such that the square sum of all these differences is not too big as in
Lemma~\ref{lemma:diff-bound-p}. This is illustrated in the second
graph in Figure~\ref{fig:match-score}.
This allows us to bound the probability on a bad score event, i.e.,
$\diff(\T, (S,\overline{S}), L) \leq 0,$ using Azuma's inequality 
in probability space 
$(\Omega_{\ul{h}}, \Sigma(\Omega_{\ul{h}}), {\bf Pr}_{\ul{h}})$
as in Section~\ref{sec:bounded-diff}.
The proof of the following lemma is rather long and shown in 
Section~\ref{sec:append-case-study}.
\begin{lemma}
\label{lemma:rho-3-cal-p}
Let $\ul{h}$ be the specific $2KL$-history that we record for a balanced
cut $(S,\bar{S})$ such that $\ul{h} \in \bar{\E}^L_1 \cap \bar{\E}^L_2$.
Let $\rho^L_3 = \frac{2}{N^{4L}}$.
Then for $K = \Omega(\frac{\ln N}{\gamma})$ and 
$KN = \Omega(\frac{\ln N \log \log N}{\gamma^2})$, for all $N \geq 4$,
$$\prob{\diff(\T, (S,\overline{S}), L) \leq 0
| \ul{h} \in \bar{\E}^{L}_2 \cap \bar{\E}^{L}_1, \bar{f} \mbox{ at random}} 
\leq \rho^L_3.$$
\end{lemma}
Eventually we compute the probability of events 
$\{\diff(\T, (S,\overline{S}), L)\leq 0\}$ in $\bar{\E}^N_1$ for all
balanced $(S, \overline{S})$ in Section~\ref{sec:product}.

\section{Proof of Theorem~\ref{thm:exp-h-p}}
\label{sec:cond-exp}
This section is dedicated to prove Theorem~\ref{thm:exp-h-p}.
We first give another definition.
\begin{definition}
\label{def:ell2}
$\E^{N-L}_1(S, \bar{S})$ is the same as $\E(X_1) \cup \ldots \cup \E(X_{N-L}) 
\cup \E(Y_1) \cup \ldots \cup \E(Y_{N-L})$ in 
the product probability space composed of distinct probability spaces 
defined over nodes $X_1, \ldots, X_{N-L}$ and $Y_1, \ldots, Y_{N-L}$ as in 
Definition~\ref{def:one-dim-p}.
\end{definition}
Hence $\bar{\E}^L_1$ and $\bar{\E}^{N-L}_1$ imply that no 
{\em bad node} event happens in the appropriate product spaces
thus defined.
We omit $(S, \bar{S})$ from $\E^L_1(S, \bar{S})$ 
and $\E^{N-L}_1(S, \bar{S})$ when it is clear from the context.
Given a balanced cut $(S, \bar{S})$, $\ul{h}$ records a history on 
the $2KL$ bits on swapped nodes $U_1, \ldots, U_L, V_1, \ldots, V_L$.
\begin{proposition}
\label{pro:sub-nodes}
Given all nodes are drawn from $\bar{\E}^N_1$, 
for any balanced cut $(S, \bar{S})$ and its particular $2KL$-history 
$\ul{h}$ that we record must satisfy the following:
$\ul{h} \in \bar{\E}^L_1(S, \bar{S})$.
\end{proposition}
\begin{proof}
Given $\bar{\E}^N_1$, we know that for all nodes $Z_1, \ldots, Z_{2N}$,
\begin{gather}
\label{eq:node-bits}
\diff(Z_i) \geq \expct{\diff(Z_i)} - K \gamma/4,
\end{gather}
simultaneously in the product probability space $(\Omega, \F, {\bf Pr})$,
where $\diff(Z_i)$ is a random variable solely determined 
by node $Z_i$'s bit vector.
In particular, for each balanced $(S, \bar{S})$, we focus on
the product probability space that is composed of distinct 
probability spaces defined over swapped nodes 
$U_1, \ldots, U_L, V_1, \ldots, V_L$ as in Definition~\ref{def:ell}.
After we reveal these $2L$ bit vectors on 
$U_j, V_j, \forall j = 1, \ldots, L$, by~(\ref{eq:node-bits}),
\begin{eqnarray}
\label{eq:swapped-node-u-2}
\diff(U_j) & \geq &  \expct{\diff(U_j)} - K \gamma/4, \forall j = 1, \ldots, L, \\
\label{eq:swapped-node-v-2}
\diff(V_j) & \geq &  \expct{\diff(V_j)} - K \gamma/4, \forall j = 1, \ldots, L.
\end{eqnarray}
Thus we have $\ul{h} \in \bar{\E}^L_1(S, \bar{S})$.
\end{proof}
\begin{definition}
We use $\bar{f}$ to denote the {\em future} of the $2(N-L)K$ random bits 
that we are going to reveal for the unswapped nodes
on a given balanced cut $(S, \bar{S})$.
Recall that once we are fixed to the probability space such that
$\E^N_1$ does not happen, we know that both $\ul{h}$ and $\bar{f}$
are confined; the following two notation are equivalent:
\begin{eqnarray*}
(\ul{h} \in \bar{\E}^{L}_1(S, \bar{S})) & \cap & 
(\bar{f} \in \bar{\E}^{N-L}_1(S, \bar{S})), \\
(\ul{h}, \bar{f}) & \in & \bar{\E}^{N}_1.
\end{eqnarray*}
\end{definition}
\begin{remark}
Another way of seeing $\bar{\E}^L_1(S, \bar{S})$ 
(with respect to a particular balanced cut $(S, \bar{S})$) 
is to view it as an event in the simple probability space 
$(\Omega, \F, {\bf Pr})$, such that we put constraints only on 
the specific $2L$ swapped nodes defined on $(S, \bar{S})$ 
while leaving the $\bar{f}$ at random.
Hence we have $\bar{\E}^N_1 \subset \bar{\E}^L_1(S, \bar{S})$
in $(\Omega, \F, {\bf Pr})$.
\end{remark}
We leave this confined space given $\bar{\E}^N_1$ for now and 
explore the following {\em expanded} subspace, where we require 
$\ul{h} \in \bar{\E}^L_1$ while leaving the 
future $\bar{f}$ at random. 
$(\Omega_{\ul{h}}, \Sigma(\Omega_{\ul{h}}), {\bf Pr}_{\ul{h}})$ 
corresponds to this {\em expanded} subspace, where
$\ul{h} \in \bar{\E}^L_1$.
This immediately implies the following lemma.
\begin{lemma}
\label{lemma:induction-p}
For a balanced cut $(S, \bar{S})$, given a particular $2KL$-history
$\ul{h} \in F_{2KL}$ on the $2L$ swapped nodes such that 
$\ul{h} \in \bar{\E}^L_1$,
\begin{gather}
\expect{\ul{h}}{\diff(\T, (S,\overline{S}), L)| 
\ul{h} \in \bar{\E}^L_1, \bar{f} \mbox { at random}}
\geq  L(N-L)K \gamma/2,
\end{gather}
where expectation is over all possible outcomes of the 
$2(N-L)K$ random bits in $\bar{f}$ in probability space 
$(\Omega_{\ul{h}}, \Sigma(\Omega_{\ul{h}}), {\bf Pr}_{\ul{h}})$.
\end{lemma}
\begin{proof}
For a balanced cut $(S, \bar{S})$, given $\ul{h} \in \bar{\E}^L_1$, 
where $\ul{h}$ records $2KL$ bits over swapped nodes 
$U_j, V_j, \forall j = 1, \ldots, L$, by Definition~\ref{def:one-dim-p},
\begin{eqnarray}
\label{eq:swapped-node-u-p}
\diff(U_j) & \geq & \expct{\diff(U_j)} - K \gamma/4, \forall j = 1, \ldots, L, \\
\label{eq:swapped-node-v-p}
\diff(V_j) & \geq & \expct{\diff(V_j)} - K \gamma/4, \forall j = 1, \ldots, L,
\end{eqnarray}
and hence $\diff(U_j) + \diff(V_j) \geq K \gamma/2, \forall j =1, \ldots, L$
by Proposition~\ref{pro:fm99}.
Thus, in $(\Omega_{\ul{h}}, \Sigma(\Omega_{\ul{h}}), {\bf Pr}_{\ul{h}})$, 
where $\bar{f}$ is {\em at random} and $\ul{h} \in \bar{\E}^L_1$, 
we have from Proposition~\ref{pro:history-p},
\begin{eqnarray*}
\label{eq:exp-sum-p}
\expect{\ul{h}}{\diff(\T, (S,\bar{S}), L)}
& = & 
(N-L)\sum_{j=1}^{L}\diff(U_j) + 
(N-L)\sum_{j=1}^{L}\diff(V_j) \\
& \geq & 
(N-L) \sum_{j=1}^{L} (\diff(U_j) + \diff(V_j)) \geq (N-L)L K\gamma/2.
\end{eqnarray*}
\end{proof}
Recall that $\bar{\E}^L_2$ is the event that no simultaneously large 
deviation happens across $2L$ individuals over their $2KL$ random bits.
\begin{corollary}
\label{cor:induction}
Given that 
$\ul{h} \in \bar{\E}^L_1 \cap \bar{\E}^L_2$, and $\bar{f}$ is at random:
\begin{gather}
\expect{\ul{h}}{\diff(\T, (S,\overline{S}), L) 
| \ul{h} \in \bar{\E}^L_1 \cap \bar{\E}^L_2, \bar{f} \mbox { at random}} 
\geq L(N-L)K \gamma/2,
\end{gather}
which holds so long as $\ul{h} \in \bar{\E}^L_1$.
\end{corollary}

We next bound $\expect{\ul{h}}{\diff(\T, (S,\overline{S}), L)}$ 
for all balanced $(S, \bar{S})$,
where $\ul{h}$ is confined in $\bar{\E}^N_1$ and $\bar{\E}^L_2$.
We now prove Theorem~\ref{thm:exp-h-p}.

\begin{proofof}{Theorem~\ref{thm:exp-h-p}}
By Proposition~\ref{pro:sub-nodes}, for each balanced cut $(S, \bar{S})$,
we have 
\begin{gather}
\ul{h} \in \bar{\E}^L_1(S, \bar{S}).
\end{gather}
Now apply Corollary~\ref{cor:induction}, given
that  $\ul{h} \in \bar{\E}^L_1(S, \bar{S}) \cap \bar{\E}^L_2$,
we immediately have the theorem.
\end{proofof}
\begin{remark}
$\diff(Z)$ is determined by node $Z$'s bit pattern, which is the same
when we observe it from every balanced cut, where it acts as a swapped 
node. Hence although we do have $O(2^n)$ balanced cuts, 
$\expect{\ul{h}}{\diff(\T, (S,\bar{S}), L)}$ for all balanced cuts 
are just determined by the $2N$ random variables 
$\diff(Z_1), \ldot,$ $ \diff(Z_{2N})$, each of which is determined by the bit vector 
of an individual in our sample.
\end{remark}

\section{Bounded Differences}
\label{sec:bounded-diff}
In order to show Lemma~\ref{lemma:rho-3-cal-p} (actual proof see 
Section~\ref{sec:append-case-study}), we prove Theorem~\ref{thm:diff-dev-p} in 
this section, where we bound the deviation of random variable 
$\diff(\T, (S,\overline{S}), L)$ for a particular balanced cut 
$(S, \bar{S})$. Recall that we let bit vector $(H_1, \ldots, H_{2KN})$ 
record the entire history of random bits that we see, 
where $(H_1, \ldots, H_{2KL})$ record the $2KL$-history
$\ul{H}^{(\ell)}$ on $2L$ swapped nodes.
First it is convenient to introduce some more notation:
For $\ell' \geq 2KL$, we begin to reveal the random bits on unswapped nodes
in $(S, \bar{S})$. The random variable
$\expect{\ul{h}}{\diff(\T, (S,\overline{S}), L)|\ul{H}^{(\ell')}}$ 
depends on the random extension  $\ul{H}^{(\ell')}$ of $\ul{h}$ observed.
By definition 
$\expect{\ul{h}}{\diff(\T, (S,\overline{S}), L)|\ul{H}^{(\ell')}}(\pi) =
\expect{\ul{h}}{\diff(\T,(S,\overline{S}),L)|\ul{H}^{(\ell')} = \ul{h}'}$
for $\pi \in \Omega_{\ul{h}}$, where $\ul{h}' = \ul{H}^{(\ell')}(\pi)$;
another notation for this is
$\expect{\ul{h}}{\diff(\T, (S,\overline{S}), L)|\F}$ where $\F$ is the 
$\sigma$-field generated by $\ul{H}^{(\ell')}$ restricted to  
$\Omega_{\ul{h}}$.
To prove the theorem, we introduce the following.
\begin{lemma}{\bf(Azuma's Inequality)}
\label{lemma:azuma}
Let $Z_0, Z_1, \ldot, Z_m = f$ be a martingale on some probability
space, and suppose that $|Z_i - Z_{i-1}| \leq c_i$, 
 $\forall i = 1, 2, \ldot, m$, then
\[
\prob{|f - \expct{f}| \geq t} \leq 2 e^{-t^2/2\sigma^2},   
\]
where $\sigma^2 = \sum_{i=1}^m c_i^2$.
\end{lemma}
We are now ready to use bounded differences approach in 
$(\Omega_{\ul{h}}, \Sigma(\Omega_{\ul{h}}), {\bf Pr}_{\ul{h}})$ and
prove Theorem~\ref{thm:diff-dev-p}.
\begin{theorem}
\label{thm:diff-dev-p}
Let $\ul{h}$ be a possible $2KL$-history that we record for a balanced
cut $(S,\bar{S})$  such that 
$\ul{h} \in \bar{\E}^{L}_2 \cap \bar{\E}^{L}_1$.
Then, for $t > 0$, in probability space 
$(\Omega_{\ul{h}}, \Sigma(\Omega_{\ul{h}}), {\bf Pr}_{\ul{h}})$, where 
all future $2(N-L)K$ random bits $\bar{f}$ are completely at random,
\[
\probb{\ul{h}}{|\expect{\ul{h}}{\diff(\T, (S,\bar{S}), L)|\ul{H}^{2KN}}
 - \expect{\ul{h}}{\diff(\T, (S,\bar{S}), L)}| \geq t} \leq 
2 e^{-t^2/2\sigma^2},
\]
where $\sigma^2 \leq 4 (N-L)L^2 (K \gamma) + 4(N-L)L \Delta$, 
for all balanced $(S, \bar{S})$ with $0< L \leq N/2$ swapped nodes.
\end{theorem}
\begin{proof}
We shall set up things to use Lemma~\ref{lemma:azuma}.
We work in probability space 
$(\Omega_{\ul{h}}, \Sigma(\Omega_{\ul{h}}), {\bf Pr}_{\ul{h}})$.
We start to reveal the $2K(N-L)$ bits 
on unswapped nodes that are chosen independently at random,
and rely on $2L$ swapped nodes having a good history
$\ul{h}$, given that $\ul{h} \in \bar{\E}^{L}_2 \cap \bar{\E}^{L}_1$.

Given the $\sigma$-field 
$(\Omega_{\ul{h}}, \Sigma(\Omega_{\ul{h}}))$, with $\Sigma(\Omega_{\ul{h}}) = 
2^{\Omega_{\ul{h}}}$, let us first define a filter $\Filter$.
Given independent random bits $H_{2KL+1}, \ldots, H_{2KN}$, the filter is 
defined by letting $\F_{i}, \forall i = 1, \ldots, m$, where $m = 2K(N-L)$, 
be the $\sigma$-field generated by histories
$\ul{H}^{(2KL + 1)},  \ldots, \ul{H}^{(2KL + i)}$.
We thus obtain a natural $\Filter$:
\begin{gather*}
\{\emptyset, \Omega_{\ul{h}} \} = \F_0 
\subset \F_1 \subset \ldots \subset \F_m =  2^{\Omega_{\ul{h}}},
\end{gather*} 
where for $0 \leq i \leq m = 2K(N-L)$, $(\Omega_{\ul{h}}, \F_i)$ 
is a $\sigma$-field.
Hence $\Filter$ corresponds to the increasingly refined partitions of 
$\Omega_{\ul{h}}$ obtained from all the different possible extensions
of the $2KL$-history $\ul{h}$. 

We obtain a martingale for random variable
$\diff(\T, (S,\bar{S}), L)$ 
such that:
Let $Z_0 = \expect{\ul{h}}{\diff(\T, (S,\bar{S}), L)}$ and 
\begin{eqnarray}
Z_{\ell' - 2KL}  =
\expect{\ul{h}}{\diff(\T, (S,\bar{S}), L)|\ul{H}^{(\ell')}}  = 
\expect{\ul{h}}{\diff(\T, (S,\overline{S}), L)|\F_{\ell'- 2KL}},
\end{eqnarray}
where $\F_{\ell' -2KL}$ is the $\sigma$-field generated by 
$\ul{H}^{(\ell')}$ restricted to $\Omega_{\ul{h}}$ and
$2KN \geq \ell' > 2KL$.
Let  $H_{2KL+1}, \ldots,$ $H_{2KN}$ map to random bits 
on $x_i^1, \ldots, x_{N-L}^K, y_i^1, \ldots y_{N-L}^K$, 
where $x_i^k$ or  $y_i^k$ refers to a single bit on dimension
$k$ on individual $X_i$ or $Y_i$ respectively.
We first define the following, $\forall j = 1, 2, \ldots, m$, where
$m = 2K(N-L)$,
\begin{eqnarray}
\size{Z_{j} - Z_{j-1}} = c_j.
\end{eqnarray}
We also need to translate between 
$c_j$, where $j = 1, 2, \ldots, m$, and $d_{i,k}(X_i)$ and
$d_{i,k}(Y_i)$, $\forall i = 1, \ldots, N-L, k =1 ,\ldots, K$
that correspond to the bit on dimension $k$ of $X_i$ and $Y_i$ 
respectively. In particular, $\forall i, \forall k$, we let
\begin{eqnarray}
c_{(i-1) K + k} & = & d_{i, k}(X_i), \\
c_{(N-L + i-1) K + k} & = & d_{i, k}(Y_i).
\end{eqnarray}
\begin{figure}[htb]
\begin{center}
\psfig{figure=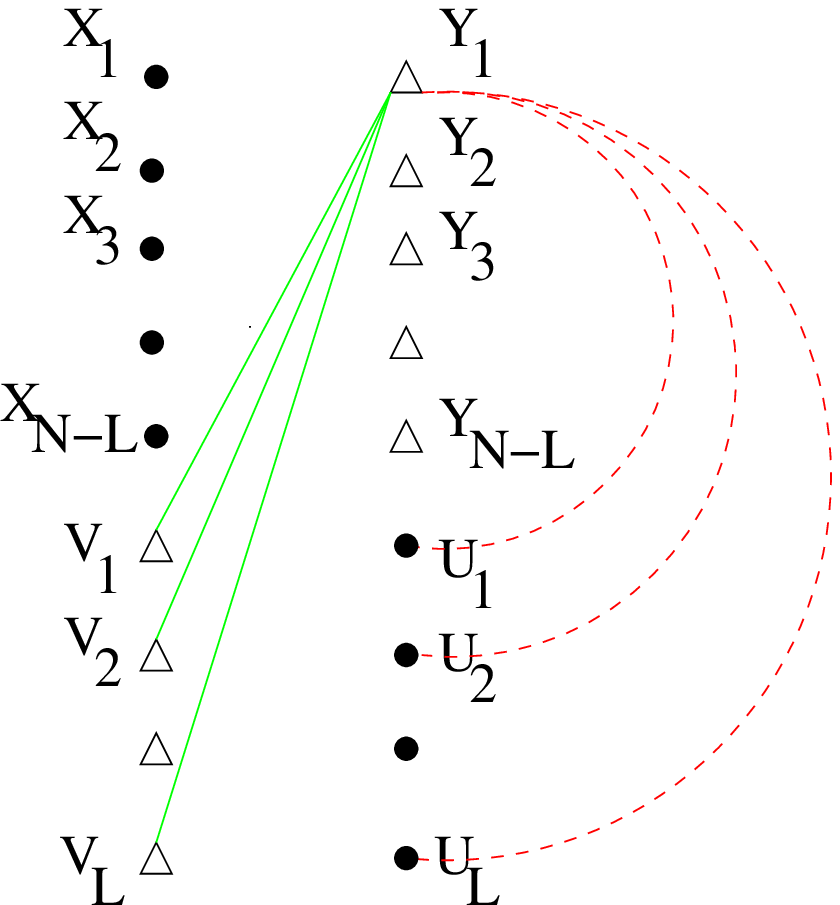, width=2in,height=2in}
\caption{Set of edges that random bits on $Y_1$ influence upon}
\label{fig:match-random}
\end{center}
\end{figure}
Let $j= 2KL + (i-1) K + k -1$, we have
\begin{eqnarray}
d_{i, k}(X_i) =  
\size{\expect{\ul{h}}{\diff(\T, (S,\bar{S}), L)|\ul{H}^{(j)}, x^k_i} 
- \expect{\ul{h}}{\diff(\T, (S,\bar{S}), L)|\ul{H}^{(j)}}}.
\end{eqnarray}
And similarly, let $\ell'= 2KL + (N-L)K + (i-1) K + k -1$, we have
\begin{eqnarray*}
d_{i, k}(Y_i) =
\size{\expect{\ul{h}}{\diff(\T, (S,\bar{S}), L)|\ul{H}^{(\ell')}, y^k_i} 
- \expect{\ul{h}}{\diff(\T, (S,\bar{S}), L)|\ul{H}^{(\ell')}}}.
\end{eqnarray*}
We immediately have the following lemma that we can plug into Azuma's 
inequality, where $d_{i, k}$ applies to both 
$d_{i, k}(X_i)$ and $d_{i, k}(Y_i)$.

\begin{lemma}
\label{lemma:diff-bound-2-p}
For the $2(N-L)K$ random bits on unswapped nodes $X_i, Y_i$ 
$\forall i \in [1, N-L]$ that we reveal, 
at dimension $k \in [1, K]$, we have
\[d_{i, k} \leq \size{L(p^k_2 - p^k_1)} + \size{t_k\sqrt{L}},\] 
where $t_k$ is defined in Definition~\ref{def:devi-p} and
$\Delta$ as in Definition~\ref{def:big-delta}, and
$\sum_{k=1}^K {t}^2_k \leq \Delta$.
\end{lemma}

\begin{proof}
Given that $Y_i, \forall i$, comes from $\DB_2$ and $X_i, \forall i$, comes
from $\DB_1$, and by definition of $d_{i,k}(Y_i)$ and $d_{i,k}(X_i)$,
\[
d_{i, k}(Y_i) = 
\left\{ \begin{array}
{r@{\quad:\quad}l} 
\size{p^k_2} \size{f^k_2(\ul{h})}
& y^k_i = 0, \\
\size{1- p^k_2}\size{f^k_2(\ul{h})}
& y^k_i = 1,\\
\end{array} 
\right. \\
\]
and
\[
d_{i, k}(X_i) = 
\left\{ \begin{array}
{r@{\quad:\quad}l} 
\size{p^k_1} \size{f^k_2(\ul{h})}
& x^k_i = 0, \\
\size{1-p^k_1}\size{f^k_2(\ul{h})}
& x^k_i = 1. \\
\end{array} 
\right.
\]
Hence given that $\ul{h} \in \bar{\E}^L_2$,
Lemma~\ref{lemma:diff-bound-p}, and
$\size{\expct{f^k_2(\ul{h})}} = \size{L(p^k_2 -p^k_1)}$
as in Proposition~\ref{pro:expect-e2-p},
\begin{eqnarray}
\label{eq:difference-1-p}
d_{i, k}(Y_i)  \leq \size{f^k_2(\ul{h})} \leq 
\size{\expct{f^k_2(\ul{h})}} + \size{t_k \sqrt{L}}  = 
\size{L(p^k_2 - p^k_1)} + \size{t_k\sqrt{L}},
\end{eqnarray}
and similarly,
$d_{i, k}(X_i) \leq \size{L(p^k_2 - p^k_1)} + \size{t_k\sqrt{L}}$,
where $\sum_{k=1}^K {t}^2_k \leq \Delta$.
\end{proof}

We are now ready to obtain a bound for 
$\sigma^2 = 2 \sum_{i=1}^{N-L} \sum_{k=1}^K d^2_{i, k}$, where
$d^2_{i, k} \leq \size{L(p^k_2 - p^k_1)} + \size{\sqrt{L}(t_k)})^2$
applies to unswapped nodes $X_i, Y_i, \forall i = 1, \ldots, N-L,$ 
in bounding the differences they cause by revealing the random
bits on dimension $K$.

Given that $\sum_{k=1}^K {t}^2_k \leq \Delta$,
\begin{eqnarray*}
\sigma^2 =  \sum_{i, k} (d^2_{i, k}(X_i) + d^2_{i, k}(Y_i)) 
& = & 2\sum_{i, k} d^2_{i, k} \leq  2 \sum_{i =1}^{N-L} \sum_{k=1}^K 
\left(\size{L(p^k_2 - p^k_1)} + \size{\sqrt{L}(t_k)}\right)^2 \\
& \leq & 
2 (N-L) \sum_{k} 2(L(p^k_2 - p^k_1))^2 + 2(\sqrt{L}(t_k))^2 \\
& = & 
4 L^2 (N-L) \sum_k (p^k_2 - p^k_1)^2 + 4 L (N-L) \sum_k t^2_k \\
& \leq &
4 (N-L)L^2 (K \gamma) + 4(N-L)L \Delta, 
\end{eqnarray*}
where $\Delta = 8N \ln2 + 4K \ln2 (\log \log N + 1) 
+ 3\ln N/2$ as in Definition~\ref{def:big-delta}.
\end{proof}

\section{Putting Things Together}
\label{sec:product}
First, there are two lemmas regarding these events.
We want to emphasize the we exclude $\bar{\E}^N_1$ once for all $2N$
nodes, while excluding one $\bar{\E}^L_2$ from each balanced cut 
$(S, \bar{S})$, where $L$ denotes that the event $\bar{\E}^L_2$ is defined 
over the particular set of $2KL$ bits across $K$ dimensions on 
the $2L$ swapped nodes in $(S, \bar{S})$; we have 
${N \choose L}^2$ number of such events for each $L$, 
whose probabilities we sum up later using union bound.
\begin{lemma}
\label{lemma:bad-event-1}
Let $K \geq \frac{256 \ln {N}}{\gamma}$, 
in probability space $(\Omega, \F, {\bf Pr})$,
$\prob{\E^N_1} \leq \rho_1 = \frac{2N}{N^{32}}$.
\end{lemma}
\begin{proof}
Apply Lemma~\ref{lemma:mu-x-bound-p} to each $\diff(Z)$ with $\tau = 1/N^{32}$; 
Given $K \geq \frac{256 \ln {N}}{\gamma}$, we have $\forall Z$,
\[
\probb{Z}{\E(Z)} \leq \inv{N^{32}}.
\]
We adopt the view of composing the product space
$(\Omega, \F, {\bf Pr})$ through distinct probability spaces 
$(\Omega_1, \F_1, {\bf Pr}_1)$, \ldots, 
$(\Omega_{2N}, \F_{2N}, {\bf Pr}_{2N})$ as in 
Definition~\ref{def:product-space}, where 
$(\Omega_i, \F_i, {\bf Pr}_i), \forall i$, is defined over all possible
outcomes for $K$ random bits for individual $Z_i$.
Therefore by definition, event $\bar{\E}^N_1$ is the same as the joint event
$\bar{\E}(Z_1) \cap \ldots \cap \bar{\E}(Z_{2N})$ in
$(\Omega, \F, {\bf Pr})$.
\begin{eqnarray}
\label{eq:equal-prob}
\prob{\bar{\E}^N_1} 
& = & \prob{\mbox{none of }\E(Z) \mbox{ happens, for all nodes }Z} \\
& = & 
\prob{\bar{\E}(Z_1) \cap \bar{\E}(Z_2) \cap \ldots \cap 
\bar{\E}(Z_{2N})} \\
& = & \probb{1}{\bar{\E}(Z_1)}\cdot \probb{2}{\bar{\E}(Z_2)} 
\cdot \ldots \cdot \probb{2N}{\bar{\E}(Z_{2N})} \\
& = & (1 - \probb{1}{\E(Z_1)}) \cdot (1-\probb{2}{\E(Z_2)}) 
\cdot \ldots \cdot (1- \probb{2N}{\E(Z_{2N})}) \nonumber\\
& \geq & (1 - \inv{N^{32}})^{2N} \geq 1 - \frac{2N}{N^{32}}.
\end{eqnarray}
\end{proof}
Before we prove Lemma~\ref{lemma:rho-2-p}, 
first let us obtain the expected value of $f^k_2(\ul{h}), \forall k$ as in 
Definition~\ref{def:f2-product}.
\begin{proposition}
\label{pro:expect-e2-p}
$\expct{f^k_2(\ul{h})} = 
\expct{\sum_{j=1}^L u^k_j - v^k_j} 
= L(p^k_1 -p^k_2).$
\end{proposition}
Next we examine the deviation 
for each random variable $f^k_2(\ul{h}), \forall k$.
\begin{lemma}
\label{lemma:unit-dev-p}
$\forall k$, for random variable $f^k_2(\ul{h})$ as in 
Definition~\ref{def:f2-product},
\begin{gather}
\prob{\size{f^k_2(\ul{h}) - \expct{f^k_2(\ul{h})}} \geq 
t_k \sqrt{L}} \leq 2e^{-{t_k}^2}.
\end{gather}
In addition, events corresponding to different dimensions are independent.
\end{lemma}
\begin{proof}
Let us define random variables  $\bar{U}^k$, $\bar{V}^k$ such that
\begin{gather}
f^k_2(\ul{h})  = L(\bar{U}^k - \bar{V}^k),
\end{gather}
where
$\bar{U}^k =  \sum_{j=1}^L u^k_j/L$ and $\bar{V}^k = \sum_{j =1}^L v^k_j/L$.
Thus by Proposition~\ref{pro:expect-e2-p},
\begin{gather*}
\expct{\bar{U}^k} - \expct{\bar{V}^k}
=  \inv{L}\expct{f^k_2(\ul{h})} = p^k_1 -p^k_2.
\end{gather*}
Now applying Corollary~\ref{coro:hoe-diff} of Theorem~\ref{thm:hoe-bound}
to bound probability of deviations on both sides of the expected 
differences, let $t = t_k \sqrt{L}/L$,
we have 
\begin{eqnarray*}
\prob{\size{f^k_2(\ul{h}) - \expct{f^k_2(\ul{h})}} \geq 
t_k \sqrt{L}} & = & \prob{\size{\bar{U}^k -\bar{V}^k - 
(\expct{\bar{U}^k} - \expct{\bar{V}^k})} 
\geq t_k \sqrt{L}/L} \\
& \leq & 2e^{\frac{ -2 (t_k \sqrt{L}/L)^2}{(2/L)}} \leq 2e^{-{t^2_k}}.
\end{eqnarray*}
\end{proof}
The following two lemmas shows that $\{\ul{h} \in \E^L_2\}$ remains 
exponentially small given $\bar{\E}^{N}_1$ or not.
A variant of the following lemma has been used in the full proof 
for~\citet[Theorem 3.1]{CHRZ07}. It is included in 
Section~\ref{sec:bad-events-append} for completeness.
\begin{lemma}{\textnormal{\citep{CHRZ07}}}
\label{lemma:rho-2-p}
In probability space $(\Omega, \F, {\bf Pr})$,
for each balanced cut $(S, \bar{S})$, \ \\
$\prob{\ul{h} \in \E^L_2} \leq \rho_2,$
where $\rho_2 = O(\inv{2^{2N}\poly(N)})$ and $N \geq 2$.
\end{lemma}
\begin{lemma}
\label{lemma:e2-e1}
$\prob{\ul{h} \in \E^{L}_2 | \bar{\E}^{N}_1} =
\prob{\ul{h} \in \E^L_2 | \ul{h} \in \bar{\E}^L_1} \leq 
\frac{\rho_2}{1-2L/N^{32}}.$
\end{lemma}
\begin{proof}
Given the following equations:
\begin{eqnarray}
\nonumber
\prob{\ul{h} \in \E^L_2}
&  = & 
\prob{\ul{h} \in \E^L_2 | \ul{h} \in {\E}^L_1} \cdot 
\prob{\ul{h} \in \E^L_1} + 
\prob{\ul{h} \in \E^L_2 | \ul{h} \in \bar{\E}^L_1}\cdot \prob{\ul{h} \in \bar{\E}^L_1}, \\ 
\prob{\ul{h} \in \bar{\E}^L_1} 
& = & (1-\inv{N^{32}})^{2L} \geq {1-2L/N^{32}},
\end{eqnarray}
we have:
\begin{eqnarray}
\prob{\ul{h} \in \E^L_2 | \ul{h} \in \bar{\E}^L_1} 
& = & \frac{\prob{\ul{h} \in \E^L_2} -
\prob{\ul{h} \in \E^L_2 | \ul{h} \in {\E}^L_1} \cdot 
\prob{\ul{h} \in \E^L_1}}
{\prob{\ul{h} \in \bar{\E}^L_1}} \\
& \leq & \frac{\prob{\ul{h} \in \E^L_2}}{\prob{\ul{h} \in \bar{\E}^L_1}} \leq
\frac{\rho_2}{1-2L/N^{32}}.
\end{eqnarray}
\end{proof}
Lemma~\ref{lemma:mod-diff} shows that
$\probb{\ul{h}}{\diff(\T, (S,\overline{S}), L) \leq 0}$
remains small regardless whether $\bar{f}$ stays in the 
confined subspace $\bar{\E}^{N}_1$ or is entirely at random as in
$(\Omega_{\ul{h}}, \Sigma(\Omega_{\ul{h}}), {\bf Pr}_{\ul{h}})$.
\begin{lemma}
\label{lemma:mod-diff}
$\prob{\diff(\T, (S,\overline{S}), L) \leq 0 |
(\ul{h}, \bar{f}) \in \bar{\E}^{N}_1 \cap
\ul{h} \in \bar{\E}^{L}_2}
\leq \frac{\rho^L_3}{1-{2(N-L)}/{N^{32}}}.$
\end{lemma}
\begin{proof}
We use $e_0$ to replace $\{\diff(\T, (S,\overline{S}), L) \leq 0\}$ 
and bound the following:
$$\prob{e_0 | (\ul{h} \in \bar{\E}^{L}_1 \cap \bar{\E}^{L}_2) 
\cap \bar{f} \in \bar{\E}^{N-L}_1},$$ 
which is the same as the term in the statement of the lemma,
\begin{eqnarray*}
\lefteqn{
\prob{e_0 | \ul{h} \in \bar{\E}^{L}_2 \cap \bar{\E}^{L}_1, 
\bar{f} \mbox{ at random}} = } \\
& & \prob{e_0| (\ul{h} \in \bar{\E}^{L}_2 \cap \bar{\E}^{L}_1) 
\cap \bar{f} \in \bar{\E}^{N-L}_1}\cdot
\prob{\bar{f} \in \bar{\E}^{N-L}_1 
| {\ul{h} \in \bar{\E}^{L}_2 \cap \bar{\E}^{L}_1}} + \\
& & 
\prob{e_0 | (\ul{h} \in \bar{\E}^{L}_2 
\cap \bar{\E}^{L}_1) \cap \bar{f} \in {\E}^{N-L}_1}\cdot
\prob{\bar{f} \in {\E}^{N-L}_1 | 
\ul{h} \in \bar{\E}^{L}_2 \cap \bar{\E}^{L}_1}. 
\end{eqnarray*}
By independence between node events: 
\begin{eqnarray}
\prob{\bar{f} \in \bar{\E}^{N-L}_1 
| {\ul{h} \in \bar{\E}^{L}_2 \cap \bar{\E}^{L}_1}} 
& = & \prob{\bar{f} \in \bar{\E}^{N-L}_1}, \\
\prob{\bar{f} \in {\E}^{N-L}_1 | 
\ul{h} \in \bar{\E}^{L}_2 \cap \bar{\E}^{L}_1}
& = & 
\prob{\bar{f} \in {\E}^{N-L}_1}.
\end{eqnarray}
Given that events ${\E}^{L}_2, {\E}^{L}_1$  defined on $2L$ swapped nodes 
are independent of event ${\E}^{N-L}_1$ on $2(N-L)$ unswapped nodes, 
we have the following, where we omit writing out the 
$\bar{f} \mbox{ at random}$ condition,
\begin{eqnarray*}
& & 
\prob{e_0 | (\ul{h} \in \bar{\E}^{L}_2 \cap \bar{\E}^{L}_1) 
\cap \bar{f} \in \bar{\E}^{N-L}_1} \nonumber\\
& = & 
{\frac{\prob{e_0| \ul{h} \in \bar{\E}^{L}_2 \cap \bar{\E}^{L}_1} 
- \prob{e_0 | (\ul{h} \in \bar{\E}^{L}_2 \cap \bar{\E}^{L}_1) 
\cap \bar{f} \in {\E}^{N-L}_1}\cdot
\prob{\bar{f} \in {\E}^{N-L}_1}}
{\prob{\bar{f} \in \bar{\E}^{N-L}_1}}}  \nonumber\\
& \leq & 
\frac{\prob{\diff(\T, (S,\overline{S}), L) \leq 0 | 
\ul{h} \in \bar{\E}^{L}_2 \cap \bar{\E}^{L}_1}}
{\prob{\bar{f} \in \bar{\E}^{N-L}_1}} \leq
\frac{\rho^L_3}{(1-\inv{N^{32}})^{2(N-L)}} \leq 
\frac{\rho^L_3}{(1-\frac{2(N-L)}{N^{32}})},
\end{eqnarray*}
where 
$\prob{\bar{f} \in \bar{\E}^{N-L}_1} 
 \geq 1-\frac{2(N-L)}{N^{32}}$ following a proof similar to
that of  Lemma~\ref{lemma:bad-event-1}.
\end{proof}

\begin{lemma}
\label{lemma:e2-e3-sum}
$\prob{\diff(\T, (S,\overline{S}), L) \leq 0 | \bar{\E}^{N}_1}
\leq \frac{\rho_2}{1-2L/N^{32}} + \frac{\rho^L_3}{1-2(N-L)/N^{32}}.$
\end{lemma}
\begin{proof}
By assumption of independence between node events,
\begin{eqnarray*}
\prob{\ul{h} \in \E^{L}_2 | \bar{\E}^{N}_1}
& =& \prob{\ul{h} \in \E^{L}_2 | 
\ul{h} \in \bar{\E}^{L}_1 \cap \bar{f} \in \bar{\E}^{N-L}_1}
 =
\prob{\ul{h} \in \E^{L}_2 | \ul{h} \in  \bar{\E}^{L}_1} \leq  
\frac{\rho_2}{1-2L/N^{32}}.
\end{eqnarray*}
When $\ul{h} \in \E^{L}_2$, we give up bounding 
$\diff(\T, (S,\overline{S}), L) \leq 0$; hence by Lemma~\ref{lemma:e2-e1}
and~\ref{lemma:mod-diff},
\begin{eqnarray*} 
\lefteqn{\prob{\diff(\T, (S,\overline{S}), L) \leq 0 | \bar{\E}^{N}_1} 
\leq
\prob{\ul{h} \in \E^{L}_2 | \bar{\E}^{N}_1} +  } \\
& & 
\prob{\diff(\T, (S,\overline{S}), L) \leq 0 |
(\ul{h}, \bar{f}) \in \bar{\E}^{N}_1 \cap \ul{h} \in \bar{\E}^{L}_2}
\cdot \prob{\ul{h} \in \bar{\E}^{L}_2 | \bar{\E}^{N}_1}  \\ 
& \leq &  
\frac{\rho_2}{1-2L/N^{32}} + \frac{\rho^L_3}{1-2(N-L)/N^{32}},
\end{eqnarray*}
\end{proof}

Finally, we prove Theorem~\ref{thm:intro-product-max-cut}.

\begin{proofof}{Theorem~\ref{thm:intro-product-max-cut}}
\begin{eqnarray*}
\lefteqn{\prob{\exists (S, \bar{S}) \mbox{ s.t. } \score(S, \bar{S}) >  \score{\T}}
\leq } \\
& & \prob{\E^{N}_1} + \sum_{(S, \bar{S})} 
\prob{\diff(\T, (S,\overline{S}), L) \leq 0 | \bar{\E}^{N}_1} \\
& \leq & \frac{32}{N^{32}} + 
\frac{2^{2N}\rho_2}{1-2L/N^{32}}+ \sum_{L =1}^{N/2} {N \choose L}{N \choose L}
\frac{\rho^L_3}{1-2(N-L)/N^{32}} =  O\left(\inv{\poly(N)}\right)
\end{eqnarray*}
\end{proofof}

\section*{Acknowledgments}
{This material is based on research sponsored in part by the 
Army Research Office, under agreement number DAAD19--02--1--0389, 
and NSF grant CNF--0435382.
The author thanks Avrim Blum for many helpful discussions and Alon Orlitsky 
for asking the question: why is not one bit enough?}
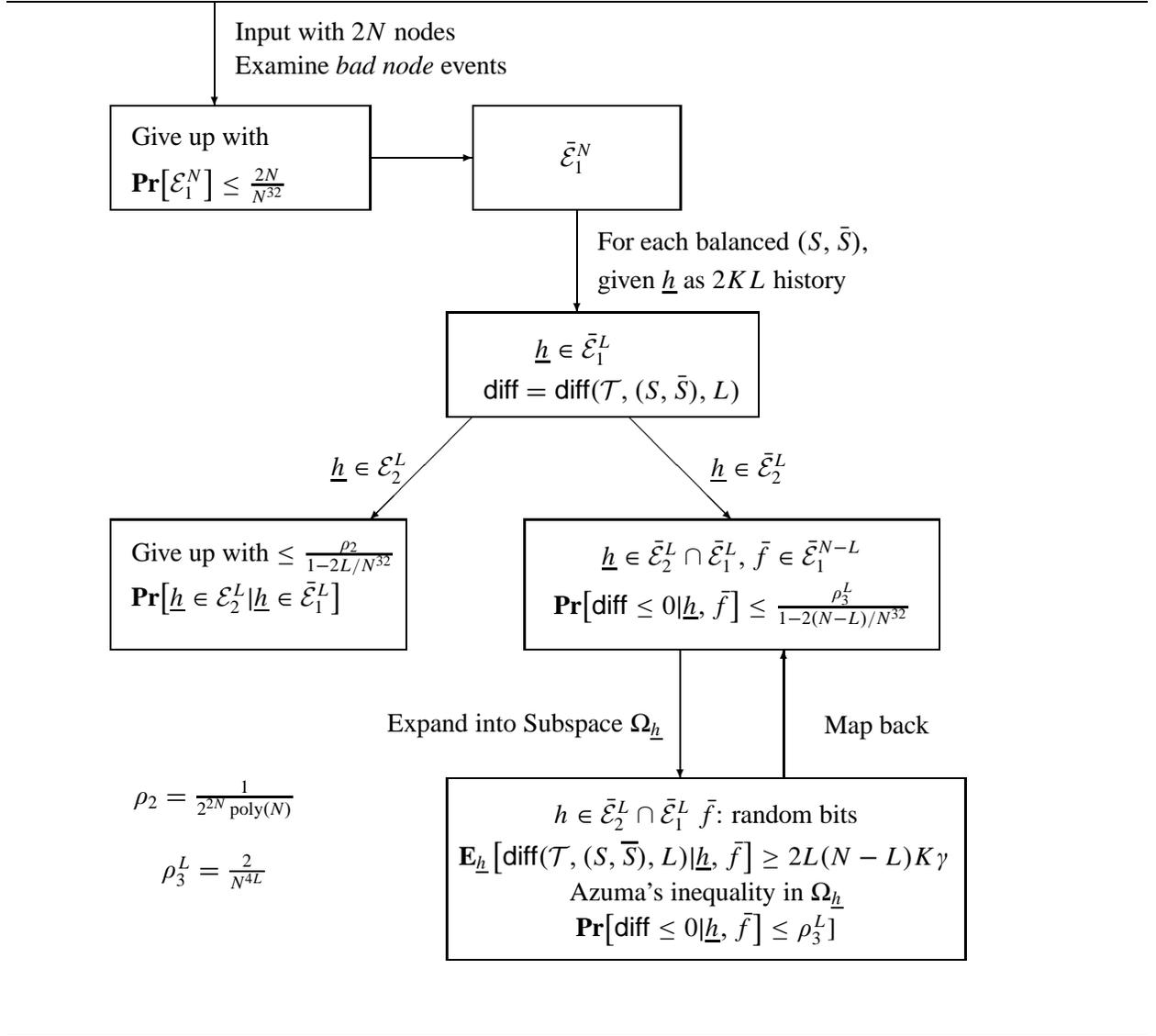
\begin{figure}[ht]
\setlength{\unitlength}{1.5cm}
\hrule
\begin{picture}(4,10)
\put(2, 10){\vector(0,-1){1}}
\put(2.2, 9.2){\makebox(2,1)[l]{Input with $2N$ nodes}}
\put(2.2, 8.9){\makebox(2,1)[l]{Examine {\em bad node} events}}

\put(1, 8){\framebox(2.5,1){}}
\put(1.2,7.8){\makebox(,)[tl]{Give up with}}
\put(1.2,7.4){\makebox(,)[tl]{$\prob{\E^N_1} \leq \frac{2N}{N^{32}}$}}
\put(3.5, 8.5){\vector(1,0){1}}

\put(4.5, 8){\framebox(2,1){$\bar{\E}^N_1$}}
\put(5.5, 8){\vector(0,-1){1}}
\put(5.7, 7.2){\makebox(2,1)[l]{For each balanced $(S, \bar{S})$,}}
\put(5.7, 6.8){\makebox(2,1)[l]{given $\ul{h}$ as $2KL$ history}}

\put(4.25, 6){\framebox(3,1){}}
\put(5.1, 5.8){\makebox(2,1)[lt]{$\ul{h} \in \bar{\E}^L_1$}}
\put(4.6, 5.4){\makebox(2,1)[lt]{$\diff = \diff(\T, (S,\bar{S}), L)$}}

\put(4.5, 6){\vector(-1,-1){1}}
\put(3.0, 5){\makebox(1,1){$\ul{h} \in \E^L_2$}}

\put(6, 6){\vector(1,-1){1}}
\put(6.8, 5){\makebox(1,1)[l]{$\ul{h} \in \bar{\E}^L_2$}}

\put(1, 3.75){\framebox(2.85, 1.25){}}
\put(1.2,3.8){\makebox(,)[tl]{Give up with $\leq \frac{\rho_2}{1-2L/N^{32}}$}}
\put(1.2,3.4){\makebox(,)[tl]
{$\prob{\ul{h} \in \E^L_2 | \ul{h} \in \bar{\E}^L_1}$}}

\put(5, 3.75){\framebox(4,1.25){}}
\put(6, 3.8){\makebox(2,1)[t]{$\ul{h} \in \bar{\E}^L_2 \cap \bar{\E}^L_1,
\bar{f} \in \bar{\E}^{N-L}_1$}}
\put(6, 3.4){\makebox(2,1)[t]{$\prob{\diff \leq 0 | \ul{h}, \bar{f}} \leq 
\frac{\rho^L_3}{1-2(N-L)/{N^{32}}}$}}

\put(6.5, 3.75){\vector(0,-1){1.25}}
\put(3, 2.5){\makebox(4,1){Expand into Subspace $\Omega_{\ul{h}}$}}

\put(7.5, 2.5){\vector(0,1){1.25}}
\put(6.4, 2.5){\makebox(4,1){Map back}}

\put(1, 1.8){\makebox(2,1){$\rho_2 = \inv{2^{2N}\poly(N)}$}}
\put(1, 1.1){\makebox(2,1){$\rho^L_3 = \frac{2}{N^{4L}}$}}

\put(4.25, 0.75){\framebox(5,1.75){}}

\put(5.75, 1.3){\makebox(2,1)[t]{$h \in \bar{\E}^L_2 \cap \bar{\E}^L_1$
$\bar{f}$: random bits}}
\put(5.75, 0.9){\makebox(2,1)[t]{$\expect{\ul{h}}{\diff(\T, (S,\overline{S}), L) 
| \ul{h}, \bar{f}} \geq 2 L(N-L)K \gamma$}}
\put(5.75, 0.5){\makebox(2,1)[t]{Azuma's inequality in $\Omega_{\ul{h}}$}}
\put(5.75, 0.2){\makebox(2,1)[t]
{$\prob{\diff \leq 0|\ul{h},\bar{f}} \leq \rho^L_3]$}}
\end{picture}
\hrule
\caption{\sc Events Relationship in Section~\ref{sec:product}}
\label{fig:events}
\end{figure}

\bibliography{final}
\newpage
\appendix
\section{Proof of Lemma~\ref{lemma:rho-2-p}}
\label{sec:bad-events-append}
The following proof have been used in the full proof 
in~\citet[Theorem 3.1]{CHRZ07}.
\ \\
\begin{proofof}{Lemma~\ref{lemma:rho-2-p}}
To facilitate our proof, we obtain a set of nonnegative 
numbers $(\tilde{t}_1, \ldots, \tilde{t}_k)$ as follows; 
$\forall k$, to obtain $\tilde{t}_k$, we round $|t_k|$ 
down to nearest nonnegative number $|\tilde{t_k}|$ that is power of two.
It is easy to verify that 
$\forall k, t_k \in 
\left[\frac{-2L - \expct{f^k_2(\ul{h})}}{\sqrt{L}}, 
\frac{2L -\expct{f^k_2(\ul{h})}}{\sqrt{L}}\right]$ by
Proposition~\ref{pro:expect-e2-p}. Thus we have 
$\tilde{t}_k \leq \size{t_k} 
\leq \size{2\sqrt{L}} + 
\size{\frac{\expct{f^k_2(\ul{h})}}{\sqrt{L}}}$.
Let us divide the entire range of $\size{t_k}$ into intervals using
power-of-$2$ non-negative integers as dividing points; 
Let $r_k, \forall k$ represent the number of such intervals: we have
$\forall k$, so long as $N \geq 2$,
\begin{eqnarray}
\label{eq:n-range-p}
r_k = \log (\size{\sqrt{L}} + \size{L(p^k_1 -p^k_2)/\sqrt{L}}) 
\leq  \log 2 \sqrt{L} \leq \log 2\sqrt{N/2} \leq \log N.
\end{eqnarray}
Thus we have at most ${(\log N)}^K$ blocks in the $K$-dimensional space
such that each block along each dimension is a subinterval of 
$\left[0, \size{2\sqrt{L}} + \size{\frac{\expct{f^k_2(\ul{h})}}{\sqrt{L}}}
\right]$.
Let $\Ball(\beta_1, \ldot, \beta_k)$ represent a block in
the $K$-dimensional space, where $\beta_1, \ldot, \beta_k$ are 
nonnegative power-of-2 integers and every point in 
$\Ball(\beta_1, \ldot, \beta_k)$ has its value fixed in interval
$[\beta_k, 2\beta_k)$ along dimension $k, \forall k$; 
hence $({\beta_1}, \ldot, {\beta_k})$ is the point in the $K$-dimensional 
space with the smallest coordinate in every dimension in 
$\Ball(\beta_1, \ldot, \beta_k)$.

A set of values $(t_1, \ldot, t_k)$ as in Definition~\ref{def:devi-p} 
is mapped into one of these blocks uniquely as follows.
We say a point $(t_1, \ldots, t_k)$ {\em maps to} 
$\Ball(\beta_1, \ldot, \beta_k)$, if
$\forall k, 2{\beta_k} > \size{t_k} \geq {\beta_k}$,
i.e., $(\tilde{t_1}, \ldot, \tilde{t_k}) = (\beta_1, \ldot, \beta_k)$.
We first bound the following event using Lemma~\ref{lemma:one-block-p}.
Let us fix one block $\Ball(\beta_1, \ldot, \beta_k)$ for 
a fixed set of values $\beta_1, \ldot, \beta_k$ such that
$\sum_{k=1}^K \beta^2_k \geq \Delta/4$.
\begin{lemma}
\label{lemma:one-block-p}
Let $\Delta/4 = 2N \ln2 + K (\ln2) (\log \log N + 1) + (3\ln N)/8$ 
as $\Delta$ is defined in Definition~\ref{def:big-delta}.
$$\prob{\ul{h} \mbox { maps to a fixed } \Ball(\beta_1, \ldot, \beta_k)
\mbox { s.t.} 
\sum_{k=1}^K \tilde{t}^2_k \geq \Delta/4}
\leq \inv{2^{2N}\cdot {(\log N)}^K \cdot N^{3/2}}.$$
\end{lemma}
\begin{proof}
Let $t_1\sqrt{L}, \ldot, t_k\sqrt{L}$ be the deviation
that we observe in $\ul{h}$ for random variables 
$f^1_2(\ul{h}), f^2_2(\ul{h}), \ldots, f^k_2(\ul{h})$ as in 
Definition~\ref{def:devi-p}.
If coordinates $(\tilde{t_1}, \ldot, \tilde{t_k})$ of 
$\ul{h}$ maps to $(\beta_1, \ldot, \beta_k)$, we know that 
$\forall k, 2 {\beta_k} \geq \size{t_k} \geq {\beta_k}$ 
given the definition of $\Ball(\beta_1, \ldot, \beta_k)$.
In addition, by Lemma~\ref{lemma:unit-dev-p}, we know that
\begin{gather}
\prob{\size{f^k_2(\ul{h}) - \expct{f^k_2(\ul{h})}} \geq 
{\beta_k} \sqrt{L}} \leq 2e^{-{\beta_k}^2/4},
\end{gather}
and events corresponding to different dimensions are independent; Thus we have
\begin{eqnarray}
& & 
\prob{\ul{h} \mbox { maps to a particular } \Ball(\beta_1, \ldot, \beta_k)
\mbox { s.t.} 
\sum_{k=1}^K \beta^2_k \geq \Delta/4} \nonumber\\
& = &   
\prod_{k=1}^K 
\prob{2{\beta_k} \sqrt{L}
\geq 
\left(\size{f^k_2(\ul{h}) - \expct{f^k_2(\ul{h})}} 
= \size{t_k \sqrt{L}}\right) 
\geq {\beta_k} \sqrt{L}
\mbox { s.t.} 
\sum_{k=1}^K \beta^2_k \geq \Delta/4} \nonumber\\
& \leq &
\prod_{k=1}^K 
\prob{\size{f^k_2(\ul{h}) - \expct{f^k_2(\ul{h})}} \geq 
{\beta_k} \sqrt{L}
\mbox { s.t.} 
\sum_{k=1}^K \beta^2_k \geq \Delta/4} \leq 
\prod_{k =1}^K 2e^{-\beta^2_k/4} \leq  
2^K e^{-\frac{\sum_{k =1}^k \beta^2_k}{4}} \\
& \leq & 
2^K e^{-\Delta/16} \leq
2^K \exp{-(2N \ln2 + K \ln2 (\log \log N + 1) + 3\ln N/2)}\\
& = & \frac{2^K} {2^{2N}\cdot {(2\log N)}^K \cdot N^{3/2}} 
=  \inv{2^{2N}\cdot {(\log N)}^K \cdot N^{3/2}}.
\end{eqnarray}
\end{proof}
Given that $t^2_k \leq 4 \tilde{t}^2_k, \forall k$, 
we know that $\sum_{k=1}^K {t}^2_k \geq \Delta$
implies that 
$\sum_{k=1}^K \tilde{t}^2_k 
\geq \inv{4} \sum_{k=1}^K {t}^2_k \geq \Delta/4.$
Thus we have 
\begin{eqnarray}
\prob{\sum_{k=1}^K {t}^2_k \geq \Delta}  & \leq &  
\prob{\sum_{k=1}^K \tilde{t}^2_k \geq \Delta/4} \\
& = & \prob{\ul{h} \mbox{ maps to some } \Ball(\beta_1, \ldot, \beta_k) 
                 \mbox{ s.t. }
\sum_{k=1}^K \beta^2_k \geq \Delta/4}.
\end{eqnarray}
This allows us to upper bound $\prob{\E^L_2}$ with events regarding
$\sum_{k=1}^K \tilde{t}^2_k$ as follows:
\begin{eqnarray}
\prob{\E^L_2} 
& = & 
\prob{\bigcap_{k=1}^K 
(f^k_2(\ul{h}) - \expct{f^k_2(\ul{h})} = t_k \sqrt{L}) 
\mbox{ s.t. }\sum_{k=1}^K {t}^2_k \geq \Delta} \\
& \leq & 
\prob{\ul{h} \mbox{ maps to some } \Ball(\beta_1, \ldot, \beta_k) 
\mbox{ s.t. }
\sum_{k=1}^K \beta^2_k \geq \Delta/4} \nonumber  \\
& \leq & \frac{(\log N)^K}{2^{2N}\cdot {(\log N)}^K \cdot N^{3/2}} 
\leq \inv{2^{2N}\poly(N)}.
\end{eqnarray}
Hence the probability that 
the $2KL$ unordered pairs induce simultaneously large deviation 
for random variables $f^1_2(\ul{h}), \ldots, f^k_2(\ul{h})$, as 
in Definition~\ref{def:big-delta}, is at most 
 $\rho_2 = O(\inv{2^{2N}\poly(N)})$.
\end{proofof}

\subsection{Actual Proof of Lemma~\ref{lemma:rho-3-cal-p}}
\label{sec:append-case-study}
Note that the constant in the lemma has not been optimized. 

\begin{proofof}{Lemma~\ref{lemma:rho-3-cal-p}}
We take
prep$t = \expect{\ul{h}}{\diff(\T, (S,\bar{S}), L)} \geq KL(N-L)\gamma/2$ 
and plug in Theorem~\ref{thm:diff-dev-p}, we have the following:
\begin{eqnarray}
\label{eq:t-bound-p}
& & \prob{\diff(\T, (S,\overline{S}), L) \leq 0
| \ul{h} \in \bar{\E}^{L}_2 \cap \bar{\E}^{L}_1} \nonumber \\
& = &
\probb{\ul{h}}{\expect{\ul{h}}{\diff(\T, (S,\bar{S}), L)|\ul{H}^{2KN}}
- \expect{\ul{h}}{\diff(\T, (S,\bar{S}), L)} \leq -
 \expect{\ul{h}}{\diff(\T, (S,\bar{S}), L)}} \nonumber \\
& \leq & 
2e^{-t^2/2\sigma^2} \leq 2 e^{-(KL(N-L)\gamma/2)^2/2\sigma^2},
\end{eqnarray}
where $\sigma^2 \leq 4 (N-L)L^2 (K \gamma) + 4(N-L)L \Delta$ as defined in 
Theorem~\ref{thm:diff-dev-p}.

We will prove that for all $N \geq 4$, so long as
\begin{enumerate}
\item
\label{eq:k-lower} 
$K \geq \Omega(\frac{\ln N}{\gamma})$,
\item
\label{eq:kn-lower}
$KN \geq \Omega(\frac{\ln N \log \log N}{\gamma^2})$,
\end{enumerate}
we will have
\begin{gather}
\label{eq:prob-bound}
2e^{-t^2/2\sigma^2} \leq 2e^{-(2KL(N-L)\gamma)^2/2\sigma^2} 
\leq \frac{2}{N^{4L}}.
\end{gather}

In what follows, we show that given different values of $N$,
by choosing slightly different constants in~(\ref{eq:k-lower}) 
and~(\ref{eq:kn-lower-2}),~(\ref{eq:prob-bound}) is always
satisfied. 
 
{\noindent 
{\bf Case 1: $4 \leq N \leq \log \log N/2\gamma$.}}

In this case, we require that
$KN \geq \frac{c_1 \ln N \log \log N} {\gamma^2}$, 
where $c_1 \geq 1488$, which immediately implies the 
following inequalities given that
$N \leq \log \log N/2\gamma$:
\begin{enumerate}
\item
$K \geq \frac{2c_1 \ln N}{\gamma}$,
\item
$N \leq \frac{K \log \log N}{4c_1 \ln N}$,
\item
$\log \log N \geq 4 \gamma, \forall N \geq 4$, i.e., 
we consider cases where $\gamma$ is small enough,
\item
$\ln N \geq 2\ln 2$, $\forall N \geq 4$.
\end{enumerate}

We first derive the following term that appears in $\sigma^2$ as
specified in Theorem~\ref{thm:diff-dev-p},
\begin{eqnarray*}
16L(N-L) (32N \ln2 + 6\ln N)
& \leq &
512 \ln 2 (N-L) L N + 96 (N-L) L \ln N \\
& \leq & 
\frac{128 \ln 2 K(N-L)L \log \log N}{c_1 \ln N} + 
\frac{48 \gamma K(N-L)L}{c_1} \\
& \leq & 
\frac{64 K(N-L)L \log \log N}{c_1} + 
\frac{12 K (N-L)L \log \log N}{c_1} \\
& \leq & \frac{76 K(N-L)L \log \log N}{c_1} 
\leq K (N-L)L \log \log N, 
\end{eqnarray*}
given that $c_1 \geq 1488$.
Next, given that $L\gamma \leq N\gamma/2 \leq \frac{\log \log N}{4}$,
we have
\begin{eqnarray*}
\sigma^2 &\leq &
64 K (N-L)L (L \gamma) + 355 K (N-L)L \log \log N + 
K L(N-L) \log \log N \\
 & \leq &
16 K L (N-L) \log \log N + 356 KL(N-L) \log \log N \\
& \leq & 372 KL (N-L) \log \log N.
\end{eqnarray*}
Finally, given that $KN \geq \frac{1488 \log \log N \ln N}{\gamma^2}$,
we have:
\begin{eqnarray*}
2e^{-t^2/2\sigma^2} 
\leq e^{-(2KL(N-L)\gamma)^2/2\sigma^2}  \leq  
2e^{-\frac{4KL(N-L)\gamma^2}{2 \times 284 \log \log N}} 
\leq 2e^{-\frac{L KN \gamma^2}{284 \log \log N}}  \leq
\frac{2}{N^{4L}}.
\end{eqnarray*}

Thus we also have
$K \geq \frac{2 c_1 \ln N}{\gamma} = \frac{2976 \ln N}{\gamma}$ 
given that $N \leq \log \log N/2 \gamma$.

{\noindent 
{\bf Case 2: 
$\frac{\log \log N}{2\gamma} < 
N \leq \frac{K \log \log N}{20}$.}}

In this case, $K$ and $N$ are close and 
we require the following,
\begin{enumerate}
\item
\label{eq:k-lower-2}
$K \geq \frac{c_2 \ln N}{\gamma}$, where $c_2 = 512$,
\item
\label{eq:kn-lower-2}
$KN \geq \frac{c_0 \ln N \log \log N}{\gamma^2}$, where $c_0 = 2000$.
\end{enumerate}

Note that constants $c_0, c_2$ above are not optimized;
given any $N$, an optimal combination of $c_0, c_2$ will 
result in the lowest possible $K$ given that
$K \geq \max\{\frac{c_0 \ln N \log \log N}{N \gamma^2}, 
\frac{c_2 \ln N}{\gamma}\}$.

Given that $N \leq \frac{K \log \log N}{20}$, we have:
\begin{eqnarray*}
16L(N-L) (32N \ln2 + 6\ln N)\leq 
\frac{400}{20} K (N-L) L \log \log N \leq 20 K (N-L)L \log \log N,
\end{eqnarray*}
and hence
\begin{eqnarray*}
\sigma^2 & \leq &
64 K (N-L)L^2\gamma + 355 K (N-L)L \log \log N + 
20 K (N-L)L \log \log N \\
& \leq &
64 (N-L)L^2 K \gamma + 375 KL(N-L)\log \log N.
\end{eqnarray*}

The following inequalities are due to~(\ref{eq:k-lower-2})
and~(\ref{eq:kn-lower-2}) respectively,
\begin{eqnarray}
\frac{(2KL(N-L)\gamma)^2}{2*64 K (N-L) L^2 \gamma} 
& \geq & 16 L \ln N, \\
\frac{(2KL(N-L)\gamma)^2}{2*375 KL(N-L)\log\log N} 
& \geq & \frac{16}{3} L \ln N,
\end{eqnarray}

and thus
\begin{eqnarray}
2\sigma^2 \leq
\frac{(2KL(N-L)\gamma)^2}{16 L \ln N} + 
\frac{(2KL(N-L)\gamma)^2}{16L \ln N/3} \leq
\frac{(2KL(N-L)\gamma)^2}{4L \ln N/3},
\end{eqnarray}
and 
$2e^{-t^2/2\sigma^2} \leq 
2e^{\frac{-(2KL(N-L)\gamma)^2}{2\sigma^2}}
\leq 2e^{- 4L \ln N} \leq 2/N^{4L}$.

{\noindent {\bf Case 3: $N \geq \frac{K \log \log N}{20} \geq 16$}}.

Here we require that 
$K = \frac{c_3 \ln N}{\gamma}$ for some $c_3$ to be determined.
Thus we have $K N \geq \frac{c_3^2 \ln^2 N \log \log N}{80 \gamma^2}$,
which satisfies the constraint of the form
$KN \geq \Omega(\frac{\ln N \log \log N}{\gamma^2})$ as in 
other cases.

Given that $N \geq 4$, we have that $\ln N \geq 2\ln 2$ and hence
\begin{eqnarray*}
16L(N-L) (32N \ln2 + 6\ln N)
& \leq &
128 (N-L) L N \ln N + 6 N L (N-L) \ln N \\
& \leq & 134 (N-L)L N \ln N.
\end{eqnarray*}

Given that $K \log \log N \leq 20 N$, we have:
\begin{eqnarray*}
\sigma^2 & \leq &
64 K (N-L)L^2 \gamma + 512 \ln 2 * (K \log \log N) (N-L) L + 
134(N-L)L N \ln N \\
& \leq &
64 (N-L)L^2 (K \gamma) + 512\ln2 * 20 N (N-L)L + 
102(N-L)L N \ln N \\
& \leq &
64 \left(\frac{c_3 \ln N}{\gamma}\right)\gamma(N-L)L (N/2) + 
(N-L)L N \ln N (128 * 20 + 134) \\
& \leq &
(32 c_3 + 2694) (N-L)L N \ln N.
\end{eqnarray*}

By taking $c_3 = 188$ such that 
$c_3^2 \geq 4 (32 c_3 + 2694)$, we have 
\begin{eqnarray*}
t^2/2\sigma^2
& \geq & \frac{(2K(N-L)L\gamma)^2}{2\sigma^2} 
=  \frac{(2 c_3 (N-L)L \ln N)^2}{2\sigma^2} \geq
\frac{2 (c_3 (N-L)L \ln N)^2}
{(32 c_3 + 2694) N (N-L)L \ln N} \\
& \geq &
\frac{2 c_3^2 (N-L) L \ln N}{(32 c_3 + 2694) N} 
\geq  \frac{c_3^2 L \ln N}{(32 c_3 + 2694)} \geq 4 L \ln N.
\end{eqnarray*}
Thus $2e^{-t^2/2\sigma^2} \leq  
2 e^{-\frac{c_3^2 L \ln N}{(32 c_3 + 2694)}}
\leq 2 e^{-4 L \ln N} = \frac{2}{N^{4L}}$.
In summary, we have the following requirements.
Note that $N$ always falls into one of these cases.
For all cases, we require that $K \geq \Omega(\ln N/\gamma)$
(which is implicit for Case~$1$); the constant that we require
in $K$ for Case~$2$ is larger than that for Case~$3$, 
(i.e., $c_2 \geq c_3$ as in above), so that the two cases 
can overlap.
\begin{itemize}
\item
{\noindent 
{\bf Case 1: $16 \leq N \leq \log \log N/2\gamma$.}}
We require that $KN \geq \frac{1488 \ln N \log \log N} {\gamma^2}$, 
which implies that $K \geq 2976 \ln N/\gamma$.
\item
{\noindent 
{\bf Case 2: 
$\frac{\log \log N}{2\gamma} < N \leq \frac{K \log \log N}{20}$.}}
We require that $K \geq \frac{512 \ln N}{\gamma}$, and
$KN \geq \frac{2000 \ln N \log\log N}{\gamma^2}$.
\item
{\noindent {\bf Case 3: $N \geq \frac{K \log \log N}{20}$}}.
We require $K \geq \frac{188 \ln N}{\gamma}$.
\end{itemize}
\end{proofof}

\vskip 0.2in
\end{document}